\documentclass{article}[11pt] 
\usepackage[vmargin=1.3in,hmargin=0.8in]{geometry}

\usepackage{amsmath,amssymb,amsfonts,graphicx,amsthm,nicefrac,mathtools,bm}
\usepackage{hyperref} 
\usepackage[numbers]{natbib}
\usepackage[english]{babel}
\usepackage{times}
\usepackage[T1]{fontenc}
\usepackage{bbm,mdframed}
\usepackage[dvipsnames]{xcolor}
\usepackage{xspace}
\usepackage{rotating,multirow} 
\usepackage{tikz}
\usetikzlibrary{arrows}
\usetikzlibrary{shapes} 
\usepackage{algpseudocode,algorithm,algorithmicx} 
\usepackage{fancyhdr} 
\usepackage{tabularx}
\usepackage{subcaption}

\usepackage{float}
\usepackage{soul}
\usepackage{color}
\usepackage{enumitem}
\usepackage{comment}
\date{}
\author{\theauthor}
\title{\thetitle} 
\fancypagestyle{plain}{%
  \fancyhf{}%
  \lhead{\thepage}
}

\newcommand{\theauthor}{}
\newcommand{\thetitle}{L-Shapley and C-Shapley:\\ Efficient Model Interpretation for Structured Data}

\newcommand*\nbhd[2]{\mathcal N_{#1}(#2)}

\newcommand{\Prob}{\ensuremath{\mathbb{P}}}
\newcommand{\X}{\mathcal{X}}
\newcommand{\Y}{\mathcal{Y}}
\newcommand{\Exp}{\ensuremath{\mathbb{E}}}
\newcommand{\Exs}{\Exp}
\newcommand{\ExpModel}{\ensuremath{\Exp_m}}
\newcommand{\DisExpModel}{\ensuremath{\hat{\Exp}_m}}
\newcommand{\real}{\ensuremath{\mathbb{R}}}
\newcommand{\ProbModel}{\ensuremath{\Prob_m}}
\newcommand{\DisProbModel}{\ensuremath{\hat{\Prob}_m}}

\newcommand{\defn}{\ensuremath{: \, = }}

\newcommand{\quoting}[1]{\ensuremath{\mbox{``\emph{#1}''}}}

\makeatletter
\long\def\@makecaption#1#2{
        \vskip 0.8ex
        \setbox\@tempboxa\hbox{\small {\bf #1:} #2}
        \parindent 1.5em  %% How can we use the global value of this???
        \dimen0=\hsize
        \advance\dimen0 by -3em
        \ifdim \wd\@tempboxa >\dimen0
                \hbox to \hsize{
                        \parindent 0em
                        \hfil 
                        \parbox{\dimen0}{\def\baselinestretch{0.96}\small
                                {\bf #1.} #2
                                } 
                        \hfil}
        \else \hbox to \hsize{\hfil \box\@tempboxa \hfil}
        \fi
        }
\makeatother

\makeatletter
\newcommand\footnoteref[1]{\protected@xdef\@thefnmark{\ref{#1}}\@footnotemark}
\makeatother

\newcommand\independent{\protect\mathpalette{\protect\independenT}{\perp}}
\def\independenT#1#2{\mathrel{\rlap{$#1#2$}\mkern2mu{#1#2}}}
\newtheorem{theorem}{Theorem}
\newtheorem{definition}{Definition}
\newtheorem{lemma}{Lemma}
\makeatletter \renewcommand*{\@fnsymbol}[1]{\ensuremath{\ifcase#1\or
    \dagger\or \dagger\or \ddagger\or \mathsection\or
    \mathparagraph\or \|\or **\or \dagger\dagger \or \ddagger\ddagger
    \else\@ctrerr\fi}} \makeatother

\newcommand{\xbar}{\ensuremath{\bar{x}}}
\newcommand{\widgraph}[2]{\includegraphics[keepaspectratio,width=#1]{#2}}

\begin{document}

\author{ Jianbo Chen$^{*}$ \hspace{3mm} Le
  Song$^{\ddag,\S}$\hspace{3mm} Martin J. Wainwright$^{*,
    \diamond}$\hspace{3mm} Michael I. Jordan$^*$ } 

    \date{University
  of California, Berkeley$^{*}$\\ Georgia Institute of
  Technology$^\ddag$\\ Ant Financial$^\S$ \qquad Voleon
  Group$^\diamond$ } 

  \maketitle

% \begin{center}

% {\bf{\LARGE{\thetitle}}}

% \vspace*{.5in} 
% {\large{
% \begin{tabular}{cccc}
% Jianbo Chen$^{1}$\footnote{Work done partially during an internship at Ant Financial} & Le Song$^{2,3}$ & Martin J. Wainwright$^1$ & Michael I. Jordan$^1$\\
% \end{tabular}
% }}

% \vspace*{.2in}

% \begin{tabular}{c}
% University of California, Berkeley$^1$\\
% Georgia Institute of Technology$^2$\\
% Ant Financial$^3$  
% \end{tabular}

% \vspace*{.2in}  
% \end{center}

% \maketitle
% \lhead{Learning to Explain}
% \rhead{J. Chen, L. Song, M. J. Wainwright, M.}
\begin{abstract} 
We study instancewise feature importance scoring as a method for model
interpretation.  Any such method yields, for each predicted instance,
a vector of importance scores associated with the feature vector.
Methods based on the Shapley score have been proposed as a fair way of
computing feature attributions of this kind, but incur an exponential
complexity in the number of features.  This combinatorial explosion
arises from the definition of the Shapley value and prevents these
methods from being scalable to large data sets and complex models. We
focus on settings in which the data have a graph structure, and the
contribution of features to the target variable is well-approximated
by a graph-structured factorization.  In such settings, we develop two
algorithms with linear complexity for instancewise feature importance
scoring.  We establish the relationship of our methods to the Shapley
value and another closely related concept known as the Myerson value
from cooperative game theory. We demonstrate on both language and
image data that our algorithms compare favorably with other methods
for model interpretation.
\end{abstract}

% \textbf{Keywords:} Model Interpretation, Feature Selection, Mutual Information.

%------------------------------------------------------------------------ 
\section{Introduction} 

%------------------------------------------------------------------------

\setlength{\abovedisplayskip}{3pt}
\setlength{\abovedisplayshortskip}{1pt}
\setlength{\belowdisplayskip}{3pt}
\setlength{\belowdisplayshortskip}{1pt}
\setlength{\jot}{3pt}
\setlength{\textfloatsep}{3pt}  

Modern machine learning models, including random forests, deep neural
networks, and kernel methods, can produce high-accuracy prediction in
many applications.  Often however, the accuracy in prediction from
such black box models, comes at the cost of interpretability.  Ease of
interpretation is a crucial criterion when these tools are applied in
areas such as medicine, financial markets, and criminal justice; for
more background, see the discussion paper by
Lipton~\cite{lipton2016mythos} as well as references therein.

In this paper, we study instancewise feature importance scoring as a
specific approach to the problem of interpreting the predictions of
black-box models.  Given a predictive model, such a method yields, for
each instance to which the model is applied, a vector of importance
scores associated with the underlying features.  The instancewise
property means that this vector, and hence the relative importance
of each feature, is allowed to vary across instances. Thus, the
importance scores can act as an explanation for the specific instance,
indicating which features are the key for the model to make its
prediction on that instance.

There is now a large body of research focused on the problem of
scoring input features based on the prediction of a given instance
(for instance, see the
papers~\cite{shrikumar2016not,bach2015pixel,ribeiro2016should,lundberg2017unified,vstrumbelj2010efficient,baehrens2010explain,datta2016algorithmic,sundararajan2017axiomatic}
as well as references therein).  Of most relevance to this paper is a
line of recent
work~\cite{vstrumbelj2010efficient,lundberg2017unified,datta2016algorithmic}
that has developed methods for model interpretation based on Shapley
value~\cite{shapley1953value} from cooperative game theory.  The
Shapley value was originally proposed as an axiomatic characterization
of a fair distribution of a total surplus from all the players, and
can be applied in to predictive models, in which case each feature is
modeled as a player in the underlying game.  While the Shapley value
approach is conceptually appealing, it is also computationally
challenging: in general, each evaluation of a Shapley value requires
an exponential number of model evaluations.  Different approaches to
circumventing this complexity barrier have been proposed, including
those based on Monte Carlo
approximation~\cite{vstrumbelj2010efficient,datta2016algorithmic} and
methods based on sampled least-squares with
weights~\cite{lundberg2017unified}.

In this paper, we take a complementary point of view, arguing that the
problem of explanation is best approached within a model-based
paradigm.  In this view, explanations are cast in terms of a model,
which may or may not be the same model as used to fit the
data. Criteria such as Shapley value, which are intractable to compute
when no assumptions are made, can be more effectively computed or
approximated within the framework of a model. We focus specifically on
settings in which a graph structure is appropriate for the data;
specifically, we consider simple chains and grids, appropriate for
time series and images, respectively. We propose two measures for
instancewise feature importance scoring in this framework, which we
term \emph{L-Shapley} and \emph{C-Shapley}; here the abbreviations
``L" and ``C" refer to ``local'' and ``connected,'' respectively.  By
exploiting the underlying graph structure, the number of model
evaluations is reduced to linear---as opposed to exponential---in the
number of features. We demonstrate the relationship of these measures
with a constrained form of Shapley value, and we additionally relate
C-Shapley with another solution concept from cooperative game theory,
known as the Myerson value~\cite{myerson1977graphs}.  The Myerson
value is commonly used in graph-restricted games, under a local
additivity assumption of the model on disconnected subsets of
features. Finally, we apply our feature scoring methods to several
state-of-the-art models for both language and image data, and find
that our scoring algorithms compare favorably to several existing
sampling-based algorithms for instancewise feature importance scoring.

The remainder of this paper is organized as follows.  We begin in
Section~\ref{SecBackground} with background and set-up for the problem
to be studied.  In Section~\ref{SecMethods}, we describe the two
methods proposed and analyzed in this paper, based on the L-Shapley
and C-Shapley scores.  Section~\ref{sec:prop} is devoted to a study of
the relationship between these scores and the Myerson value.  In
Section~\ref{SecExperiments}, we evaluate the performance of L-Shapley
and C-Shapley on various real-world data sets, and we conclude with a
discussion in Section~\ref{SecDiscussion}.

%%%%%%%%%%%%%%%%%%%%%%%%%%%%%%%%%%%%%%%%%%%%%%%%%%%%%%%%%%%%%%%%%%%%%%%%%%%%%%%

\section{Background and preliminaries}
\label{SecBackground}

We begin by introducing some background and notation.

%%%%%%%%%%%%%%%%%%%%%%%%%%%%%%%%%%%%%%%%%%%%%%%%%%%%%%%%%%%%%%%%%%%%%%%%%%%%%%%%%
\subsection{Importance of a feature subset}
\label{sec:importance}

We are interested in studying models that are trained to perform
prediction, taking as input a feature vector \mbox{$x \in \X \subset
  \real^d$} and predicting a response or output variable $y \in \Y$.
We assume access to the output of a model via a conditional
distribution, denoted by $\ProbModel(\cdot |x)$, that provides the
distribution of the response $Y \in \Y$ conditioned on a given vector
$X = x$ of inputs.  For any given subset $S \subset \{1, 2, \ldots, d
\}$, we use $x_S = \{ x_j, j \in S \}$ to denote the associated
sub-vector of features, and we let $\ProbModel(Y \mid x_S)$ denote the
induced conditional distribution when $\ProbModel$ is restricted to
using only the sub-vector $x_S$.  In the cornercase in which $S =
\emptyset$, we define \mbox{$\ProbModel(Y \mid x_\emptyset)\defn
  \ProbModel(Y)$.}  In terms of this notation, for a given feature
vector $x \in \X$, subset $S$ and fitted model distribution
$\ProbModel(Y \mid x)$, we introduce the \emph{importance score}
\begin{align*}
v_x(S) \defn \ExpModel \left[ -\log\frac{1}{\ProbModel(Y \mid x_S)} \;
  \Big| \; x \right],
\end{align*}
where $\ExpModel[\cdot \mid x]$ denotes the expectation over
$\ProbModel(\cdot \mid x)$.  The importance score $v_x(S)$ has a
coding-theoretic interpretation: it corresponds to the negative of the
expected number of bits required to encode the output of the model
based on the sub-vector $x_S$.  It will be zero when the model makes a
deterministic prediction based on $x_S$, and larger when the model
returns a distribution closer to uniform over the output space.

There is also an information-theoretic interpretation to this
definition of importance scores, as discussed in our previous work~\cite{chen2018learning}.  In particular, suppose that for a given
integer $k<d$, there is a function $x \mapsto S^*(x)$ such that, for
all almost all $x$, the $k$-sized subset $S^*(x)$ maximizes $v_x(S)$
over all subsets of size $k$.  In this case, we are guaranteed that
the mutual information $I(X_{S^*(X)},Y)$ between $X_{S^*(X)}$ and $Y$
is maximized, over any conditional distribution that generates a
subset of size $k$ given $X$. The converse is also true.

In many cases, class-specific importance is favored, where one is
interested in seeing how important a feature subset $S$ is to the
predicted class, instead of the prediction as a conditional
distribution. In order to handle such cases, it is convenient to
introduce the degenerate conditional distribution
\begin{align*}
\DisProbModel(y\mid x) & \defn \begin{cases} 1 \text{ if } y \in \arg
  \max \limits_{y'} \ProbModel(y'\mid x), \\ 0 \text{ otherwise.}
\end{cases}
\end{align*}
We can then define the importance of a subset $S$ with respect to
$\DisProbModel$ using the modified score
\begin{align*}
 v_x(S) &\defn \DisExpModel \left[ -\log\frac{1}{\ProbModel(Y \mid
     x_S)} \; \Big| \; x \right],
 \end{align*} 
which is the expected log probability of the predicted class given the
features in $S$.

%%%%%%%%%%%%%%%%%%%%%%%%%%%%%%%%%%%%%%%%%%%%%%%%%%%%%%%%%%%%%%%%%%%%%%%%%%%%%

\paragraph{Estimating the conditional distribution:}

In practice, we need to estimate---for any given feature vector $\xbar
\in \X$---the conditional probability functions $\ProbModel(y \mid
\xbar_S)$ based on observed data.  Past work has used one of two
approaches: either estimation based on empirical
averages~\cite{vstrumbelj2010efficient}, or plug-in estimation using a
reference point~\cite{datta2016algorithmic,lundberg2017unified}. \\

\noindent \emph{Empirical average estimation}: In this approach, we
first draw a set of feature vector $\{x^j\}_{j=1}^M$ by sampling with
replacement from the full data set.  For each sample $x^j$, we define
a new vector $\tilde x^j \in \real^d$ with components
\begin{align*}
  (\tilde x_j)_i & = \begin{cases} \xbar_i & \mbox{if $i \in S$, and} \\
    x^j_i & \mbox{otherwise.}
  \end{cases}
\end{align*}
Taking the empirical mean of $\ProbModel (y\mid \tilde x^j)$ over
$\{\tilde x^j\}$ is then used as an estimate of $\ProbModel (y\mid
\xbar_S)$.\\

\noindent \emph{Plug-in estimation}: In this approach, the first step
is to specify a reference vector $x^0\in\real^d$ is specified. We then
define the vector $\tilde x\in\real^d$ with components
\begin{align*}
  (\tilde x)_i & = \begin{cases} \xbar_i & \mbox{if $i \in S$, and} \\
    x^0_i & \mbox{otherwise}.
  \end{cases}
\end{align*}
Finally, we use the conditional probability $\ProbModel (y\mid \tilde
x)$ as an approximation to $\ProbModel (y\mid \xbar_S)$. The plug-in
estimate is more computationally efficient than the empirical average
estimator, and works well when there exist appropriate choices of
reference points. We use this method for our experiments, where we use
the index of padding for language data, and the average pixel strength
of an image for vision data. 
%%%%%%%%%%%%%%%%%%%%%%%%%%%%%%%%%%%%%%%%%%%%%%%%%%%%%%%%%%%%%%%%%%%%%%%%%%%%%%

\subsection{Shapley value for measuring interaction between features}
\label{sec:shapleyvalue}

Consider the problem of quantifying the importance of a given feature
index $i$ for feature vector $x$.  A naive way of doing so would be by
computing the importance score $v_x(\{i\})$ of feature $i$ on its own.
However, doing so ignores interactions between features, which are
likely to be very important in applications.  As a simple example,
suppose that we were interested in performing sentiment analysis on
the following sentence:
\begin{align}
\tag{$\star$}\label{example}
\emph{It is not heartwarming or entertaining. It just sucks.}
\end{align}
This sentence is contained in a movie review from the IMDB movie data
set \cite{maas2011learning}, and it is classified as negative
sentiment by a machine learning model to be discussed in the
sequel. Now suppose we wish to quantify the importance of feature
\quoting{not} in prediction. The word \quoting{not} plays an important
role in the overall sentence as being classified as negative, and thus
should be attributed a significant weight.  However, viewed in
isolation, the word \quoting{not} has neither negative nor positive
sentiment, so that one would expect that $v_x(\{\quoting{not}\})
\approx 0$.

Thus, it is essential to consider the interaction of a given feature
$i$ with other features. For a given subset $S$ containing $i$, a
natural way in which to assess how $i$ interacts with the other
features in $S$ is by computing the difference between the importance
of all features in $S$, with and without $i$.  This difference is
called the \emph{marginal contribution} of $i$ to $S$, and given by
\begin{align}
m_x(S,i) \defn v_x(S) - v_x(S\setminus\{i\}).
\end{align}
In order to obtain a simple scalar measure for feature $i$, we need to
aggregate these marginal contributions over all subsets that contain
$i$.  The \emph{Shapley value}~\cite{shapley1953value} is one
principled way of doing so.  For each integer $k = 1, \ldots, d$, we
let $\mathcal{S}_k(i)$ denote the set of $k$-sized subsets that
contain $i$.  The Shapley value is obtained by averaging the marginal
contributions, first over the set $\mathcal{S}_k(i)$ for a fixed $k$,
and then over all possible choices of set size $k$:
\begin{align}
\label{eq:shapley}
\phi_x(\ProbModel,i) \defn \underbrace{\frac{1}{d}
  \sum_{k=1}^{d}}_{\mbox{Average over $k$}} \underbrace{
  \frac{1}{\binom{d-1}{k-1}} \sum_{S \in
    \mathcal{S}_k(i)}}_{\mbox{Average of over $\mathcal{S}_k(i)$}}
m_x(S,i).
\end{align}
%= \frac{1}{d} \sum_{S:i\in S\subset [d]}\frac
%{1}{\binom{d-1}{|S|-1}}m_x(S,i).
Since the model $\ProbModel$ remains fixed throughout our analysis, we
frequently omit the dependence of $\phi_x$ on $\ProbModel$, instead
adopting the more compact notation $\phi_x(i)$.

The concept of Shapley value was first introduced in cooperative game
theory~\cite{shapley1953value}, and it has been used in a line of
recent work on instancewise feature importance
ranking~\cite{vstrumbelj2010efficient,datta2016algorithmic,lundberg2017unified}.
It can be justified on an axiomatic basis~\cite{shapley1953value,
  young1985monotonic} as being the unique function from a collection
of $2^d$ numbers (one for each subset $S$) to a collection of $d$
numbers (one for each feature $i$) with the following properties:
\begin{description}[style=unboxed,leftmargin=0cm]
\item[Additivity:] The sum of the Shapley values $\sum_{i=1}^d
  \phi_x(i)$ is equal to the difference $v_x(\{1, \ldots, d\}) -
  v_x(\emptyset)$.
\item[Equal contributions:]  If $v_x(S \cup \{i\}) = v_x(S \cup \{j\})$ for all
  subsets $S$, then $\phi_x(i) = \phi_x(j)$.
\item[Monotonicity:] Given two models $\ProbModel$ and $\ProbModel$, let $m_x$ and $m_x'$ denote the associated marginal contribution functions, and let $\phi_x$ and $\phi_x'$ denote the associated Shapley values. If $m_x(S, i) \geq m'_x(S, i)$ for all subsets $S$, then we are guaranteed that $\phi_x(i) \geq \phi'_x(i)$.
\end{description}
Note that all three of these axioms are reasonable in our feature
selection context.

%%%%%%%%%%%%%%%%%%%%%%%%%%%%%%%%%%%%%%%%%%%%%%%%%%%%%%%%%%%%%%%%%%%%%%%%%%%%%%

\subsection{The challenge with computing Shapley values}

The exact computation of the Shapley value $\phi_x(i)$ takes into
account the interaction of feature $i$ with all $2^{d-1}$ subsets that
contain $i$, thereby leading to computational difficulties. Various
approximation methods have been developed with the goal of reducing
complexity. For example, \citet{vstrumbelj2010efficient} proposed to
estimate the Shapley values via a Monte Carlo approximation built on
an alternative permutation-based definition of the Shapley value.
\citet{lundberg2017unified} proposed to evaluate the model over
randomly sampled subsets and use a weighted linear regression to
approximate the Shapley values based on the collected model
evaluations.

In practice, such sampling-based approximations may suffer from high
variance when the number of samples to be collected per instance is
limited.  For large-scale predictive models, the number of features is
often relatively large, meaning that the number of samples required to
obtain stable estimates can be prohibitively large.  The main
contribution of this paper is to address this challenge in a
model-based paradigm, where the contribution of features to the
response variable respects the structure of an underlying graph.  In
this setting, we propose efficient algorithms and provide bounds on
the quality of the resulting approximation.  As we discuss in more
detail later, our approach should be viewed as complementary to
sampling-based or regresssion-based approximations of the Shapley
value.  In particular, these methods can be combined with the approach
of this paper so as to speed up the computation of the L-Shapley and
C-Shapley values that we propose.

%%%%%%%%%%%%%%%%%%%%%%%%%%%%%%%%%%%%%%%%%%%%%%%%%%%%%%%%%%%%%%%%%%%%%%%%%%%%%

\section{Methods}
\label{SecMethods}

% The exact computation of the Shapley value $\phi_x(i)$ takes into
% account the interaction of feature $i$ with all $2^{d-1}$ subsets that
% contain $i$, thereby leading to computational difficulties.  
In many applications, the features can be associated with the nodes of
a graph, and we can define distances between pairs of features based
on the graph structure.  More concretely, for sequence data (such as
language, music etc.), each feature vector $x$ can be associated with
a line graph, whereas for image data, each $x$ is naturally associated
with a grid graph.  In this section, we propose modified forms of the
Shapley values, referred to as L-Shapley and C-Shapley values, that
can be computed more efficiently than the Shapley value.  We also show
that under certain probabilistic assumptions on the marginal
distribution over the features, these quantities yield good
approximations to the original Shapley values.

More precisely, given feature vectors $x \in \real^d$, we let $G =
(V,E)$ denote a connected graph with nodes $V$ and edges $E\subset
V\times V$, where each feature $i$ is associated with a a node $i \in
V$, and edges represent interactions between features.  The graph
induces a distance function on $V \times V$, given by
\begin{align}
d_G(\ell, m) & = \mbox{number of edges in shortest path joining $\ell$
  to $m$}.
\end{align}
In the line graph, this graph distance corresponds to the number of
edges in the unique path joining them, whereas it corresponds to the
Manhattan distance in the grid graph.  For a given node $i \in V$, its
$k$-\emph{neighborhood} is the set
\begin{align}
  \nbhd{k}{i}
  & \defn \left \{ j \in V \mid d_G(i,j) \leq k \right \}
\end{align}
of all nodes at graph distance at most $k$.  See
Figure~\ref{FigGridStructure} for an illustration for the
two-dimensional grid graph.

\begin{figure}[h]
  \begin{center}
    \begin{tabular}{ccccc}
      \widgraph{0.22\textwidth}{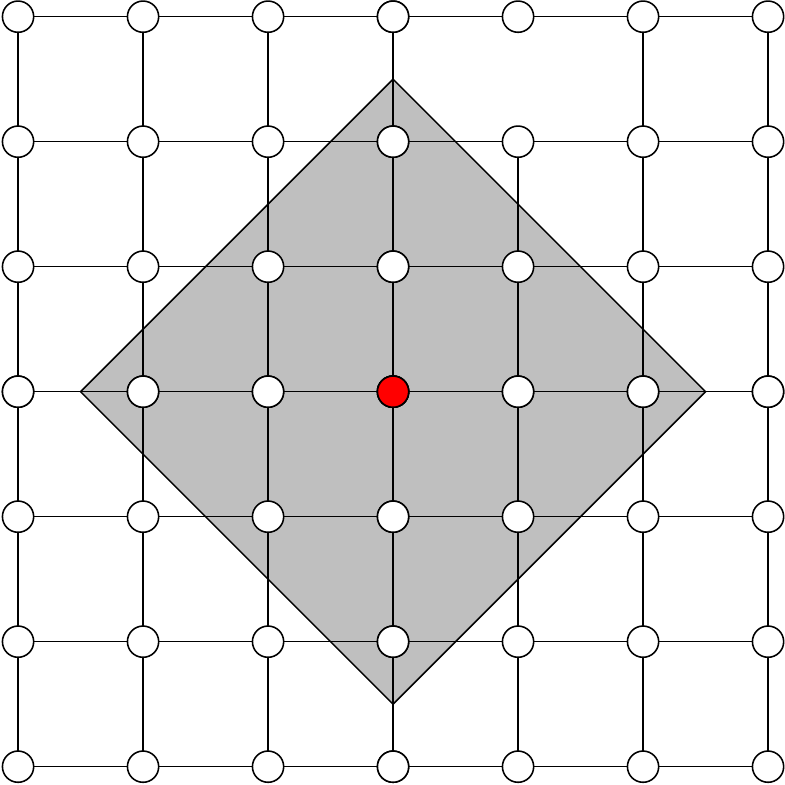} & \hspace*{.05\textwidth} &
      \widgraph{0.22\textwidth}{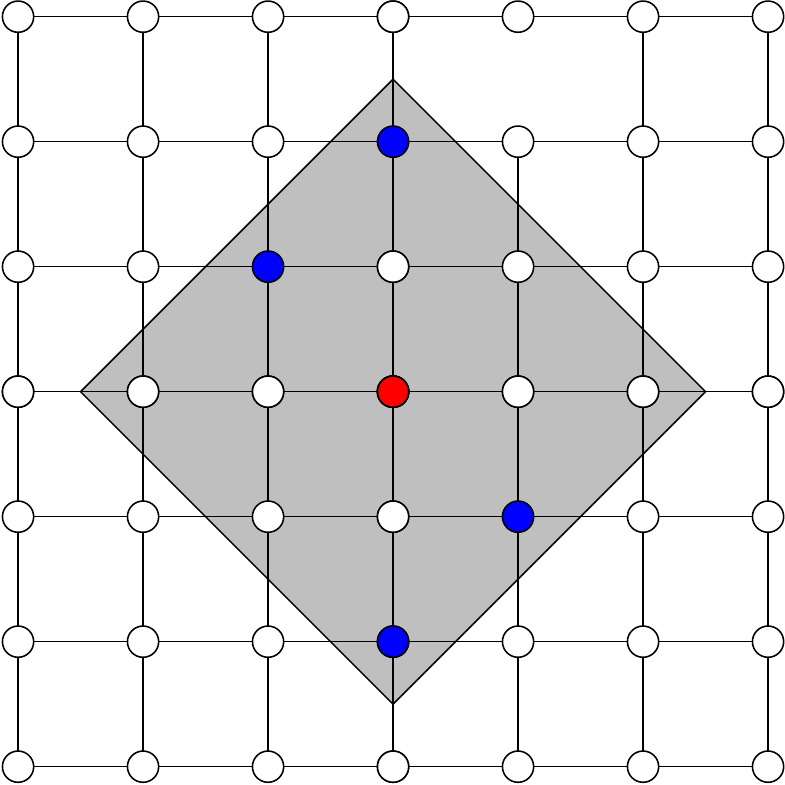} & \hspace*{.05\textwidth} &
      \widgraph{0.22\textwidth}{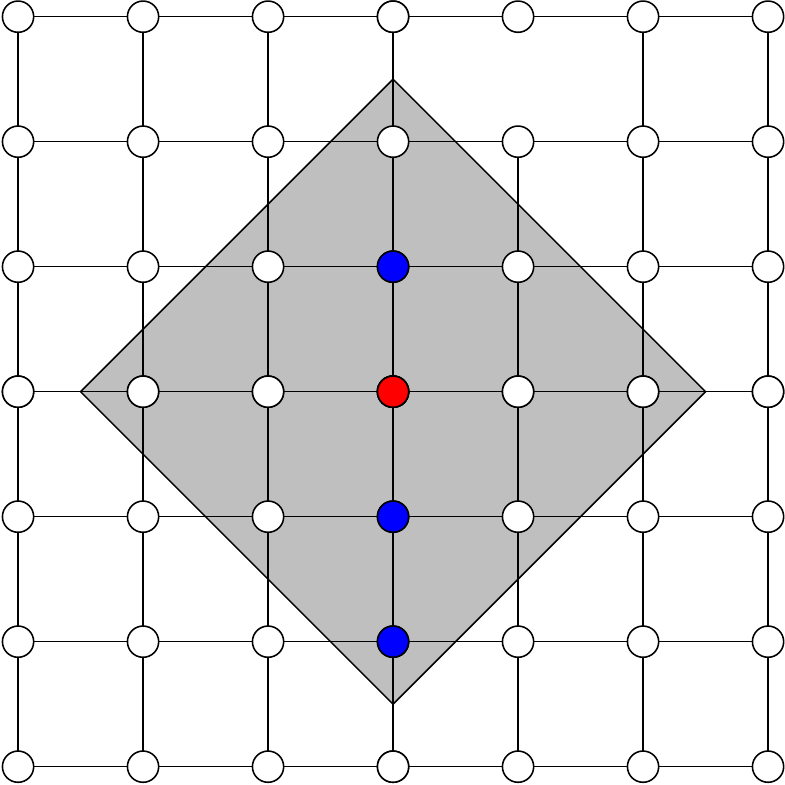} \\
      (a) && (b) && (c)
    \end{tabular}
    \caption{In all cases, the red node denotes the target feature
      $i$. (a) Illustration of the $k=2$ graph neighborhood
      $\nbhd{2}{i}$ on the grid graph.  All nodes within the shaded
      gray triangle lie within the neighborhood $\nbhd{2}{i}$. (b) A
      disconnected subset of $\nbhd{2}{i}$ that is summed over in
      L-Shapley but not C-Shapley.  (c) A connected subset of
      $\nbhd{2}{i}$ that is summed over in both L-Shapley and
      C-Shapley.}
    \label{FigGridStructure}
  \end{center}
\end{figure}

We propose two algorithms for the setting in which features that are
either far apart on the graph or features that are not directly
connected, have an accordingly weaker interaction.

\subsection{Local Shapley}

In order to motivate our first graph-structured Shapley score, let us
take a deeper look at Example~\eqref{example}. In order to compute the
importance score of \quoting{not,} the most important words to be
included are \quoting{heartwarming} and \quoting{entertaining.}
Intuitively, the words distant from them have a weaker influence on
the importance of a given word in a document, and therefore have
relatively less effect on the Shapley score.  Accordingly, as one
approximation, we propose the \text{L-Shapley} score, which only
perturbs the neighboring features of a given feature when evaluating
its importance:
\begin{definition}
\label{def:localshapley}
Given a model $\ProbModel$, a sample $x$ and a feature $i$, the
\emph{L-Shapley estimate of order $k$} on a graph $G$ is given by
\begin{align}
\hat\phi_x^k(i) & \defn \frac{1}{|\nbhd{k}{i}|} \sum_{ \substack{T \ni
    i \\T \subseteq \nbhd{k}{i} }}
\frac{1}{\binom{|\nbhd{k}{i}|-1}{|T|-1}} m_x(T,i).
\end{align}
\end{definition}
The coefficients in front of the marginal contributions of feature $i$
are chosen to match the coefficients in the definition of the Shapley
value restricted to the neighborhood $\nbhd{k}{i}$. We show in
Section~\ref{sec:prop} that this choice controls the error under
certain probabilistic assumptions. In practice, the choice of the
integer $k$ is dictated by computational considerations. By the
definition of $k$-neighborhoods, evaluating all $d$ L-Shapley scores
on a line graph requires $2^{2k} d$ model evaluations.  (In
particular, computing each feature takes $2^{2k+1}$ model evaluations,
half of which overlap with those of its preceding feature.) A similar
calculation shows that computing all $d$ L-Shapley scores on a grid
graph requires $2^{4k^2} d$ function evaluations.

%%%%%%%%%%%%%%%%%%%%%%%%%%%%%%%%%%%%%%%%%%%%%%%%%%%%%%%%%%%%%%%%%%%%%%%%%%%%%%%

%
% \mypara{Connected Shapley}
%

\subsection{Connected Shapley}

We also propose a second algorithm, C-Shapley, that further reduces
the complexity of approximating the Shapley value. Coming back to
Example~\eqref{example} where we evaluate the importance of
\quoting{not,} both the L-Shapley estimate of order larger than two
and the exact Shapley value estimate would evaluate the model on the
word subset \quoting{It not heartwarming,} which rarely appears in
real data and may not make sense to a human or a model trained on
real-world data. The marginal contribution of \quoting{not} relative
to \quoting{It not heartwarming} may be well approximated by the
marginal contribution of ``not'' to \quoting{not heartwarming.} This
motivates us to proprose \emph{C-Shapley}:
\begin{definition}\label{def:connectedshapley}
Given a model $\ProbModel$, a sample $x$ and a feature $i$, the
\emph{C-Shapley estimate of order $k$} on a graph $G$ is given by
\begin{align}
  \label{eq:connected}
\tilde\phi_x^k(i) & \defn \sum_{U \in \mathcal{C}_k(i)}
\frac{2}{(|U|+2)(|U|+1)|U|} m_x(U,i),
\end{align}
where $\mathcal{C}_k(i)$ denotes the set of all subsets of
$\nbhd{k}{i}$ that contain node $i$, and are connected in the graph
$G$.  %\mjwcomment{You also need to clarify where the weights in front
 % of the coefficients come from.}
\end{definition}
The coefficients in front of the marginal contributions are a result of using Myerson value to characterize a new coalitional game over the graph $G$, in which the influence of disconnected subsets of features are additive. The error between C-Shapley and the Shapley value can also be controlled under certain statistical assumptions. See Section~\ref{sec:prop} for details.

For text data, C-Shapley is equivalent to only evaluating n-grams in a
neighborhood of the word to be explained.  By the definition of
$k$-neighborhoods, evaluating the C-Shapley scores for all $d$
features takes $\mathcal O(k^2 d)$ model evaluations on a line graph,
as each feature takes $\mathcal O(k^2)$ model evaluations.

% Local Shapley \red{Do I need to add this? Wait, wait, need a correct
%  statement when my mind is clear.} For the plug-in estimate, we have
%  an alternative guarantee for LocalShapley, based on approximate
%  local additivity of the classifier: \begin{theorem} Suppose there
%  exist two real-valued functions $f_1,f_2:
%  \prod_{i=1}^d\{x_i,x_i^0\}\times \Y\to \real$ such that for a
%  feature subset $S\subset S$, \begin{align*} \sup_{x\subset
%  \{x^*,x^0\}^d,y} |\log\ProbModel(y|x) - (f_1(x_{S\cup \{i\}},y) +
%  f_2(x_{[d]\setminus(\{i\}\cup S)},y))|<\varepsilon.  \end{align*}
%  Then the error between the estimation $\hat\phi_i(\ProbModel,x^*)$
%  given by Eq.~\ref{def:localshapley} and $\phi_i(\ProbModel,x^*)$ is
%  bounded by $\varepsilon$: \begin{align} |\hat\phi_i(\ProbModel,x^*)
%  - \phi_i(\ProbModel,x^*)|\leq
%  \varepsilon.  \end{align} \end{theorem}

% \mjwcomment{My calculation says that it should be $(k+1)^2$
  % evaluations for each feature.  Can you clarify?}

% Putting this sentence here serves only to reduce the reader's interest
% in this definition
%However, on a grid graph, the current version of
%C-Shapley does not reduce much of the computational burden compared to
%L-Shapley. In Section~\ref{sec:related}, by using a regression-based
%approximation, the complexity of C-Shapley is further reduced to
%$\mathcal O(kd)$.

%%%%%%%%%%%%%%%%%%%%%%%%%%%%%%%%%%%%%%%%%%%%%%%%%%%%%%%%%%%%%%%%%%%%%%%%%%%%%

\section{Properties}
\label{sec:prop}

In this section, we study some basic properties of the L-Shapley and
C-Shapley values.  In particular, under certain probabilistic
assumptions on the features, we show that they provide good
approximations to the original Shapley values.  We also show their
relationship to another concept from cooperative game theory, namely
that of Myerson values, when the model satisfies certain local
additivity assumptions.

%%%%%%%%%%%%%%%%%%%%%%%%%%%%%%%%%%%%%%%%%%%%%%%%%%%%%%%%%%%%%%%%%%%%%%%%%%%

\subsection{Approximation of Shapley value}

In order to characterize the relationship between L-Shapley and the
Shapley value, we introduce \emph{absolute mutual information} as a
measure of dependence. Given two random variables $X$ and $Y$, the
absolute mutual information $I_a(X;Y)$ between $X$ and $Y$ is defined
as
\begin{align}
I_a(X;Y) = \Exs \left[ \Big|\log \frac{P(X,Y)}{P(X)P(Y)} \Big|
  \right],
\end{align}
where the expectation is taken jointly over $X,Y$. Based on the
definition of independence, we have $I_a(X;Y)=0$ if and only if
$X\independent Y$. Recall the mutual
information~\cite{cover2012elements} is defined as \mbox{$I(X;Y) =
  \Exs [\log \frac{P(X,Y)}{P(X)P(Y)}]$.}  The new measure is more
stringent than the mutual information in the sense that $I(X;Y)\leq
I_a(X;Y)$. The absolute conditional mutual information can be defined
in an analogous way. Given three random variables $X,Y$ and $Z$, we
define the absolute conditional mutual information to be $I_a(X;Y\mid
Z) = \Exs[|\log \frac{P(X,Y\mid Z)}{P(X\mid Z)P(Y\mid Z)}|]$,
where the expectation is taken jointly over $X,Y,Z$. Recall that
$I_a(X;Y\mid Z)$ is zero if and only if $X\independent Y|Z$.

Theorem~\ref{thm:localshapley} and Theorem~\ref{thm:connectedshapley}
show that L-Shapley and C-Shapley values, respectively, are related to the Shapley value whenever the model obeys a Markovian structure that is encoded by the graph.  We leave their proofs to Appendix~\ref{app:proof}. 
\begin{theorem}
\label{thm:localshapley}
Suppose there exists a feature subset $S\subset \nbhd{k}{i}$ with $i\in S$, such that 
\begin{align}
\sup_{\substack{U\subset S\setminus \{i\},
V\subset [d]\setminus S}}
I_a(X_i;X_{V}|X_U,Y)\leq\varepsilon; 
\sup_{\substack{U\subset S\setminus \{i\},
V\subset [d]\setminus S}}
I_a(X_i;X_{V}|X_U)\leq\varepsilon,
\label{eq:markov}
\end{align}
where we identify $I_a(X_i;X_{V}|X_\emptyset)$ with $I_a(X_i;X_{V})$ for notational convenience. 
Then the expected error between the L-Shapley estimate $\hat\phi_X^k(i)$ and the true Shapley-value-based importance score $\phi_i(\ProbModel,x)$ is bounded by $4\varepsilon$:
\begin{align}
\Exs_X|\hat\phi_X^k(i) - \phi_X(i)|\leq 4\varepsilon.
\end{align}
In particular, we have $\hat\phi_X^k(i) = \phi_X(i)$ almost surely if we have $X_i\independent X_{[d]\setminus S} | X_T$ and $X_i\independent X_{[d]\setminus S} | X_T,Y$ for any $T\subset S\setminus \{i\}$.
\end{theorem} 
% Theorem~\ref{thm:connectedshapley} gives an error bound for the Connected Shapley estimate: 

% \begin{theorem}
% \label{thm:connectedshapley}%\red{Proof marginal contribution, and equation 11,12 includes i. }
% Suppose there exists a neighborhood $S\subset S$ of $i$, with $i\in S$, such that Condition~\ref{eq:markov} is satisfied. Moreover, for any connected subsequences $U=\{u,u+1,\dots,u+l\}\subset S$ with $i\in U$, we have
% %and any $V\subset S$ such that $U\cap V=\emptyset$ and any element in $V$ is separated from $U$ by at least one position, we have 
% \begin{align}
% \sup_{V\in R(U)}I_a(X_i;X_{V}|X_{U\setminus \{i\}},Y)&\leq\varepsilon,\\
% \sup_{V\in R(U)}I_a(X_i;X_{V}|X_{U\setminus \{i\}})&\leq\varepsilon,
% \label{eq:markov2}
% \end{align}
% where $R(U)\defn [d] - U\cup\{\max(u-1,1),\min(u+l+1,d)\}$. Then the expected error between the Connected Shapley estimate $\tilde\phi_X^S(i)$ with respect to $S$ and the true Shapley value based importance score $\phi_i(\ProbModel,x)$ is bounded by $6\varepsilon$:
% \begin{align}
% \Exs_X|\tilde\phi_X^S(i) - \phi_X(i)|\leq 6\varepsilon,
% \end{align}
% In particular, we have $\hat\phi_X^{[d]}(i) = \phi_X(i)$ almost surely if we have $X_i\independent X_{R(U)} | X_{U\setminus \{i\}}$ and $X_i\independent X_{R(U)} | X_{U\setminus \{i\}},Y$ for any $U\subset [d]$.
% \end{theorem}

\begin{theorem}
\label{thm:connectedshapley}%\red{Proof marginal contribution, and equation 11,12 includes i. }
Suppose there exists a neighborhood $S\subset \nbhd{k}{i}$ of $i$, with $i\in S$, such that Condition~\ref{eq:markov} is satisfied. Moreover, for any connected subset $U\subset S$ with $i\in U$, we have
%and any $V\subset S$ such that $U\cap V=\emptyset$ and any element in $V$ is separated from $U$ by at least one position, we have 
\begin{align}
\sup_{V\subset R(U)}I_a(X_i;X_{V}|X_{U\setminus \{i\}},Y)\leq\varepsilon;%\\
\sup_{V\subset R(U)}I_a(X_i;X_{V}|X_{U\setminus \{i\}})\leq\varepsilon,
\label{eq:markov2}
\end{align}
 where $R(U)\defn \{i\in[d]-U:\text{ for any } j\in U, (i,j) \notin E \}$. Then the expected error between the C-Shapley estimate $\tilde\phi_X^k(i)$ and the true Shapley-value-based importance score $\phi_i(\ProbModel,x)$ is bounded by $6\varepsilon$:
\begin{align}
\Exs_X|\tilde\phi_X^k(i) - \phi_X(i)|\leq 6\varepsilon.
\end{align}
In particular, we have $\hat\phi_X^{d}(i) = \phi_X(i)$ almost surely if we have $X_i\independent X_{R(U)} | X_{U\setminus \{i\}}$ and $X_i\independent X_{R(U)} | X_{U\setminus \{i\}},Y$ for any $U\subset [d]$.
\end{theorem}

%%%%%%%%%%%%%%%%%%%%%%%%%%%%%%%%%%%%%%%%%%%%%%%%%%%%%%%%%%%%%%%%%%%%%%%%%%%%

\subsection{Relating the C-Shapley value to the Myerson value}
\label{sec:myerson}

Let us now discuss how the C-Shapley value can be related to the
Myerson value, which was introduced by~\citet{myerson1977graphs} as an
approach for characterizing a coalitional game over a graph $G$. Given
a subset of nodes $S$ in the graph $G$, let $\mathcal{C}_G(S)$ denote
the set of connected components of $S$---i.e., subsets of $S$ that are
connected via edges of the graph.  Thus, if $S$ is a connected subset
of $G$, then $\mathcal{C}_G(S)$ consists only of $S$; otherwise, it
contains a collection of subsets whose disjoint union is equal to $S$.

Consider a score function $T \mapsto v(T)$ that satisfies the
following decomposability condition: for any subset of nodes $S$, the
score $v(S)$ is equal to the sum of the scores over all the connected
components of $S$---viz.
\begin{align}
  \label{EqnDecompose}
v(S) & = \sum_{T\in \mathcal{C}_G(S)} v(T).
\end{align} 
For any such score function, we can define the associated Shapley
value, and it is known as the \emph{Myerson value} on $G$ with respect
to $v$. \citet{myerson1977graphs} showed that the Myerson value is the
unique quantity that satisfies both the decomposability property, as
well as the properties additivity, equal contributions and
monotonicity given in Section~\ref{sec:shapleyvalue}.

In our setting, if we use a plug-in estimate for conditional
probability, the decomposability condition~\eqref{EqnDecompose} is
equivalent to assuming that the influence of disconnected subsets of
features are additive at sample $x$, and C-Shapley of order $k=d$ is
exactly the Myerson value over $G$. In fact, if we partition each
subset $S$ into connected components, as in the definition of Myerson
value, and sum up the coefficients (using Lemma 1 in
Appendix~\ref{app:proof}), then the Myerson value is equivalent to
equation~\eqref{eq:connected}.

%%%%%%%%%%%%%%%%%%%%%%%%%%%%%%%%%%%%%%%%%%%%%%%%%%%%%%%%%%%%%%%%%%%%%%%%%%%%%

\subsection{Connections with related work}
\label{sec:related}

Let us now discuss connections with related work in more depth, and in
particular how methods useful for approximating the Shapley value can
be used to speed up the evaluation of approximate L-Shapley and C-Shapley
values.

\subsubsection{Sampling-based methods}

%
% {\bf Sampling-based methods}
There is an alternative definition of the Shapley value based on
taking averages over permutations of the features.  In particular, the
contribution of a feature $i$ corresponds to the average of the
marginal contribution of $i$ to its preceding features over the set of
all permutations of $d$ features. Based on this definition,
\citet{vstrumbelj2010efficient} propose a Monte Carlo approximation,
based on randomly sampling permutations.

While L-Shapley is deterministic in nature, it is possible to combine
it with this and other sampling-based methods. For example, if one
hopes to consider the interaction of features in a large neighborhood
$\nbhd{k}{i}$ with a feature $i$, where exponential complexity in $k$
becomes a barrier, sampling based on random permutation of local
features may be used to alleviate the computational burden.

%%%%%%%%%%%%%%%%%%%%%%%%%%%%%%%%%%%%%%%%%%%%%%%%%%%%%%%%%%%%%%%%%%%%%%%%%%%%

\subsubsection{Regression-based methods}

\citet{lundberg2017unified} proposed to sample feature subsets based
on a weighted kernel, and carry out a weighted linear regression to
estimate the Shapley value. Suppose the model is evaluated on $N$
feature subsets at $x$. In weighted least squares, each row of the
data matrix $X\in\{0,1\}^{N\times d}$ is a $d$-dimensional vector,
with the $j^{th}$ entry being one if the feature $j$ is selected, and
zero otherwise.  The response $F\in\real^N$ is the evaluation of the
model over feature subsets. The weight matrix $W$ is diagonal with
$W_{ii}=(d-1) / (\binom{d}{n_i}n_i(d-n_i))$ with $n_i=\sum_{j=1}^d
X_{ij}$.

\citet{lundberg2017unified} provide strong empirical results using
this regression-based approximation, referred to as KernelSHAP; in
particular, see Section~5.1 and Figure~3 of their paper.  We can
combine such a regression-based approximation with our modified
Shapley values to further reduce the evaluation complexity of the
C-Shapley values.  In particular, for a chain graph, we evaluate the
score function over all connected subsequences of length $\leq k$;
similarly, on a grid graph, we evaluate it over all connected squares
of size $\leq k\times k$.  Doing so yields a data matrix $X \in
\{0,1\}^{kd \times d}$ and a response vector $F\in\real^{kd}$, where
$X_{ij}=1$ if the $j$th feature is included in the $i$th sample, and
$F_i \defn v_x(S_i)$, the score function evaluated on the
corresponding feature subset. We use the solution to this weighted
least-squares problem as a regression-based estimate of
C-Shapley---that is, $\tilde \phi^k_x \approx (X^T W X)^{-1} X^T F$.

%%%%%%%%%%%%%%%%%%%%%%%%%%%%%%%%%%%%%%%%%%%%%%%%%%%%%%%%%%%%%%%%%%%%%%%%%
\section{Experiments}
\label{SecExperiments}

We evaluate the performance of L-Shapley and C-Shapley on real-world
data sets involving text and image classification. Codes for
reproducing the key results are available
online.\footnote{\label{ft:code}\url{https://github.com/Jianbo-Lab/LCShapley}}
We compare L-Shapley and C-Shapley with several competitive algorithms
for instancewise feature importance scoring on black-box models,
including the regression-based approximation known as
KernelSHAP~\cite{lundberg2017unified},
SampleShapley~\cite{vstrumbelj2010efficient} , and the LIME
method~\cite{ribeiro2016should}.  As discussed previously, KernelSHAP
forms a weighted regression-approximation of the Shapley values,
whereas SampleShapley estimates Shapley value by random permutation of
features. The LIME method uses a linear model to locally approximate
the original model through weighted least squares. For all methods,
the number of model evaluations is the same, and linear in the number
of features. We also choose the objective to be the log probability of
the predicted class, and use the plug-in estimate of conditional
probability across all methods (see Section~\ref{sec:importance}).

For image data, we also compare with Saliency
map~\cite{simonyan2013deep} as another baseline.  The Saliency method
is used for interpreting neural networks in computer vision, by
assuming knowledge of the gradient of a model with respect to the
input, and using the gradient magnitude as the importance score for
each pixel.

\subsection{Text Classification}

Text classification is a classical problem in natural language
processing, in which text documents are assigned to predefined
categories. We study the performance of L-Shapley and C-Shapley on
three popular neural models for text classification: word-based
CNNs~\cite{kim2014convolutional}, character-based
CNNs~\cite{zhang2015character}, and long-short term memory (LSTM)
recurrent neural networks~\cite{hochreiter1997long}, with the
following three data sets on different scales. See
Table~\ref{tab:dataset} for a summary, and Appendix~\ref{app:model}
for all of the details.

\begin{itemize} \item \textbf{IMDB Review with Word-CNN}: The Internet
  Movie Review Dataset (IMDB) is a dataset of movie reviews for
  sentiment classification \cite{maas2011learning}, which contains
  $50,000$ binary labeled movie reviews, with a split of $25,000$ for
  training and $25,000$ for testing. A simple word-based CNN model
  composed of an embedding layer, a convolutional layer, a max-pooling
  layer, and a dense layer is used, achieving an accuracy of $90.1\%$
  on the test data set.

\item \textbf{AG news with Char-CNN}: The AG news corpus is composed
  of titles and descriptions of $196,000$ news articles from $2,000$
  news sources~\cite{zhang2015character}. It is segmented into four
  classes, each containing $30,000$ training samples and $1,900$
  testing samples. Our character-based CNN has the same structure as
  that proposed in \citet{zhang2015character}. The model achieves an
  accuracy of $90.09\%$ on the test data set.

\item \textbf{Yahoo!\ Answers with LSTM}: The corpus of
  Yahoo!\ Answers Topic Classification Dataset is divided into ten
  categories, each class containing $140,000$ training samples and
  $5,000$ testing samples. Each input text includes the question
  title, content and best answer. We train a bidirectional LSTM which
  achieves an accuracy of $70.84\%$ on the test data set, close to the
  state-of-the-art accuracy of $71.2\%$ obtained by character-based
  CNNs \cite{zhang2015character}.
\end{itemize}

We choose zero paddings as the reference point for all methods, and
make $4\times d$ model evaluations, where $d$ is the number of words
for each input. Given the average length of each input (see
Table~\ref{tab:dataset}), this choice controls the number of model
evaluations under $1,000$, taking less than one second in TensorFlow
on a Tesla K80 GPU for all the three models. For L-Shapley, we are
able to consider the interaction of each word $i$ with the two
neighboring words in $\nbhd{1}{i}$ given the budget. For C-Shapley,
the budget allows the regression-based version to evaluate all
$n$-grams with $n\leq 4$.

The change in log-odds scores before and after masking the top
features ranked by importance scores is used as a metric for
evaluating performance, where masked words are replaced by zero
paddings. This metric has been used in previous literature in model
interpretation~\cite{shrikumar2016not,lundberg2017unified}. We study
how the average log-odds score of the predicted class decreases as the
percentage of masked features over the total number of features
increases on $1,000$ samples from the test set. Results are plotted in
Figure~\ref{fig:text_lor_mask}.

\begin{table}[bt!]
\centering \resizebox{0.88\textwidth}{!}{
 \begin{tabular}{||c|c|c|c|c|c|c|c||} 
 \hline
Data Set & Classes & Train Samples & Test Samples & Average \#w & Model & Parameters & Accuracy \\ [0.5ex] 
\hline\hline
IMDB Review \cite{maas2011learning} & 2 & 25,000 & 25,000 & 325.6 & WordCNN & 351,002& 90.1\% \\ 
AG's News \cite{zhang2015character} & 4 & 120,000 & 7,600 & 43.3& CharCNN & 11,337,988 & 90.09\%\\  
Yahoo!\ Answers \cite{zhang2015character} & 10 & 1,400,000 &  60,000 &108.4 & LSTM & 7,146,166 & 70.84\% \\ 
 \hline\hline
 \end{tabular} 
 } 
 \caption{A summary of data sets and models in three
   experiments. ``Average \#w'' is the average number of words per
   sentence. ``Accuracy'' is the model accuracy on test samples.}
\label{tab:dataset}

\end{table}

\begin{figure}[bt!]
\centering
\includegraphics[width=0.3\linewidth]{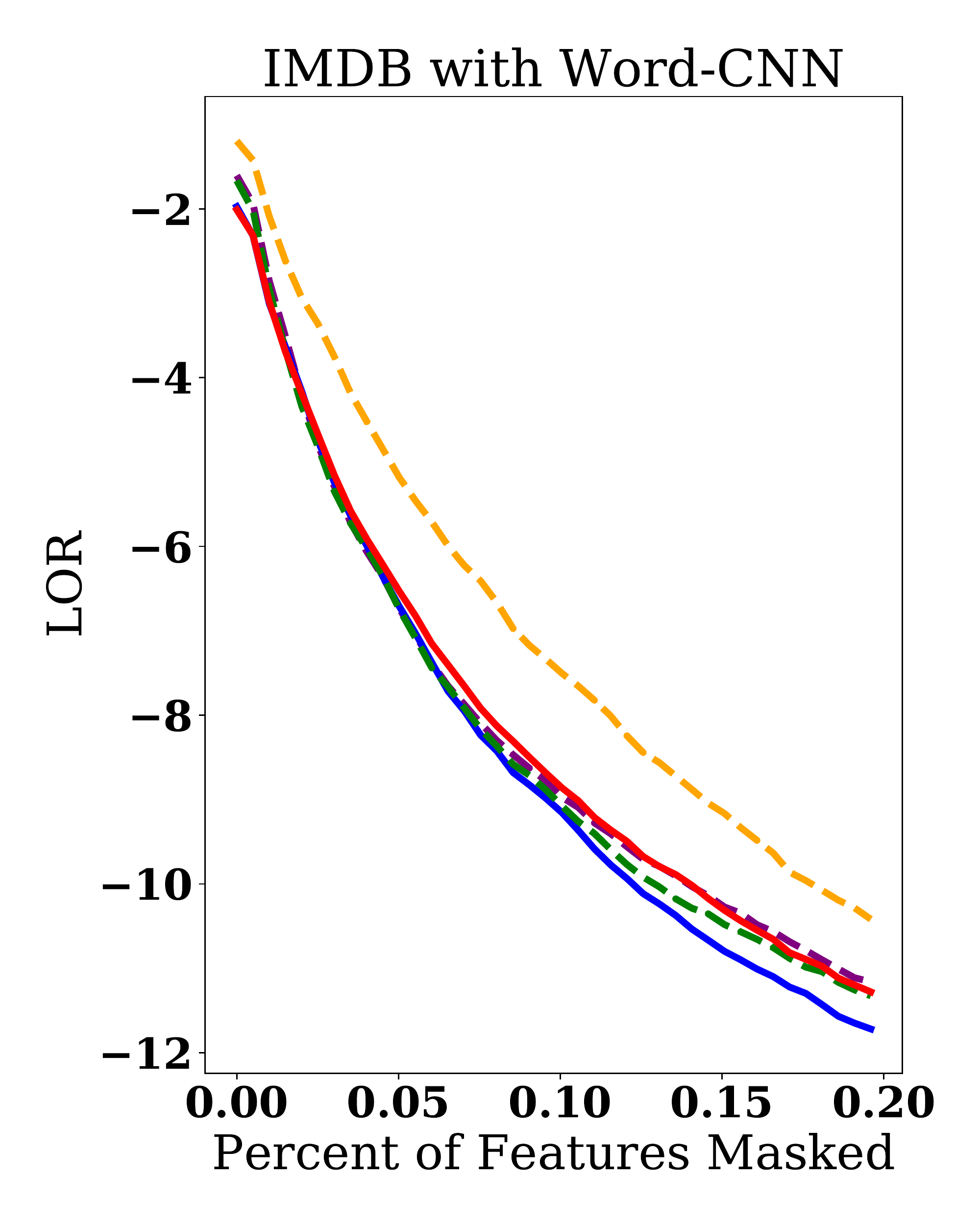}
\hspace*{-0.3cm}\includegraphics[width=0.3\linewidth]{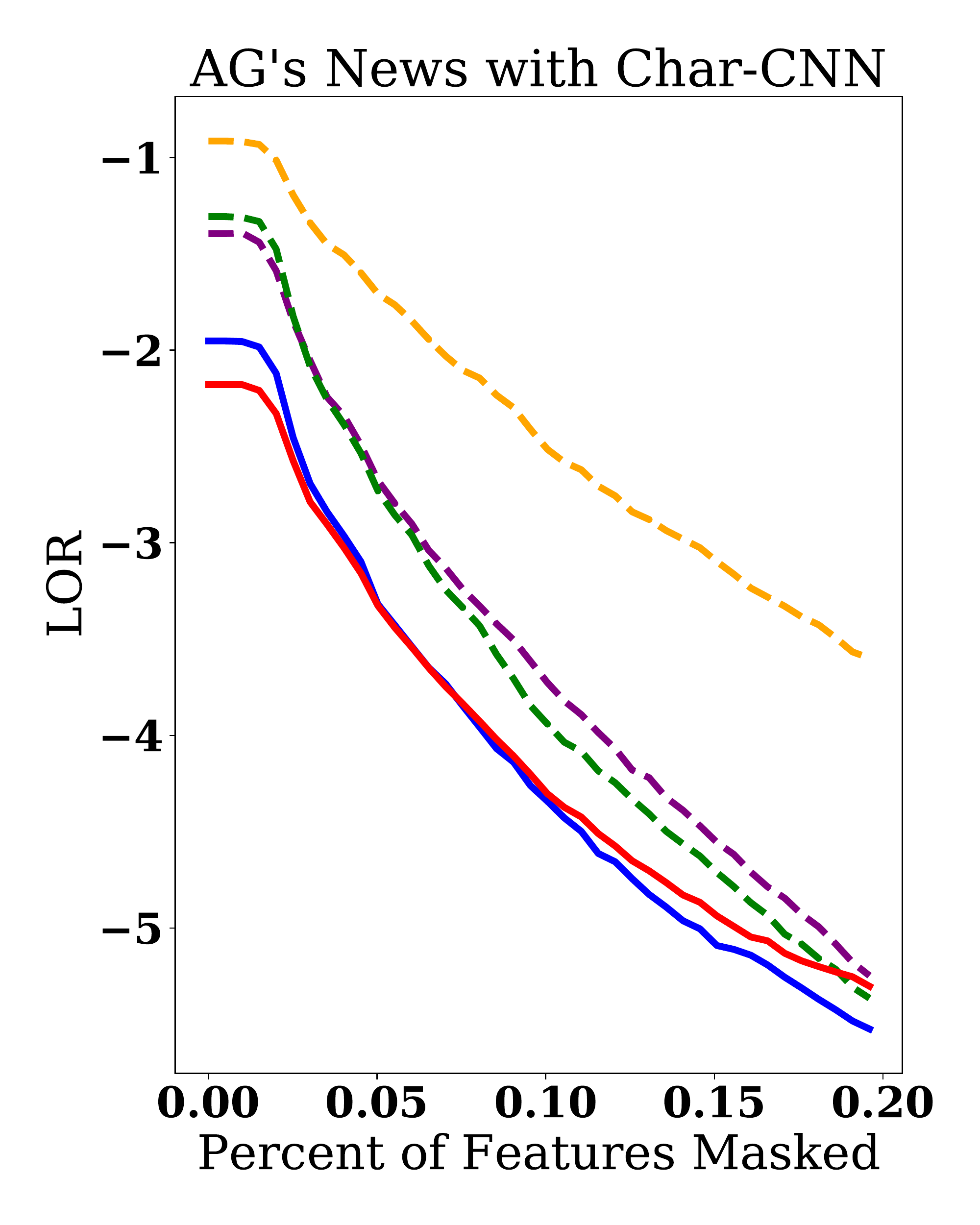}
\hspace*{-0.3cm}\includegraphics[width=0.3\linewidth]{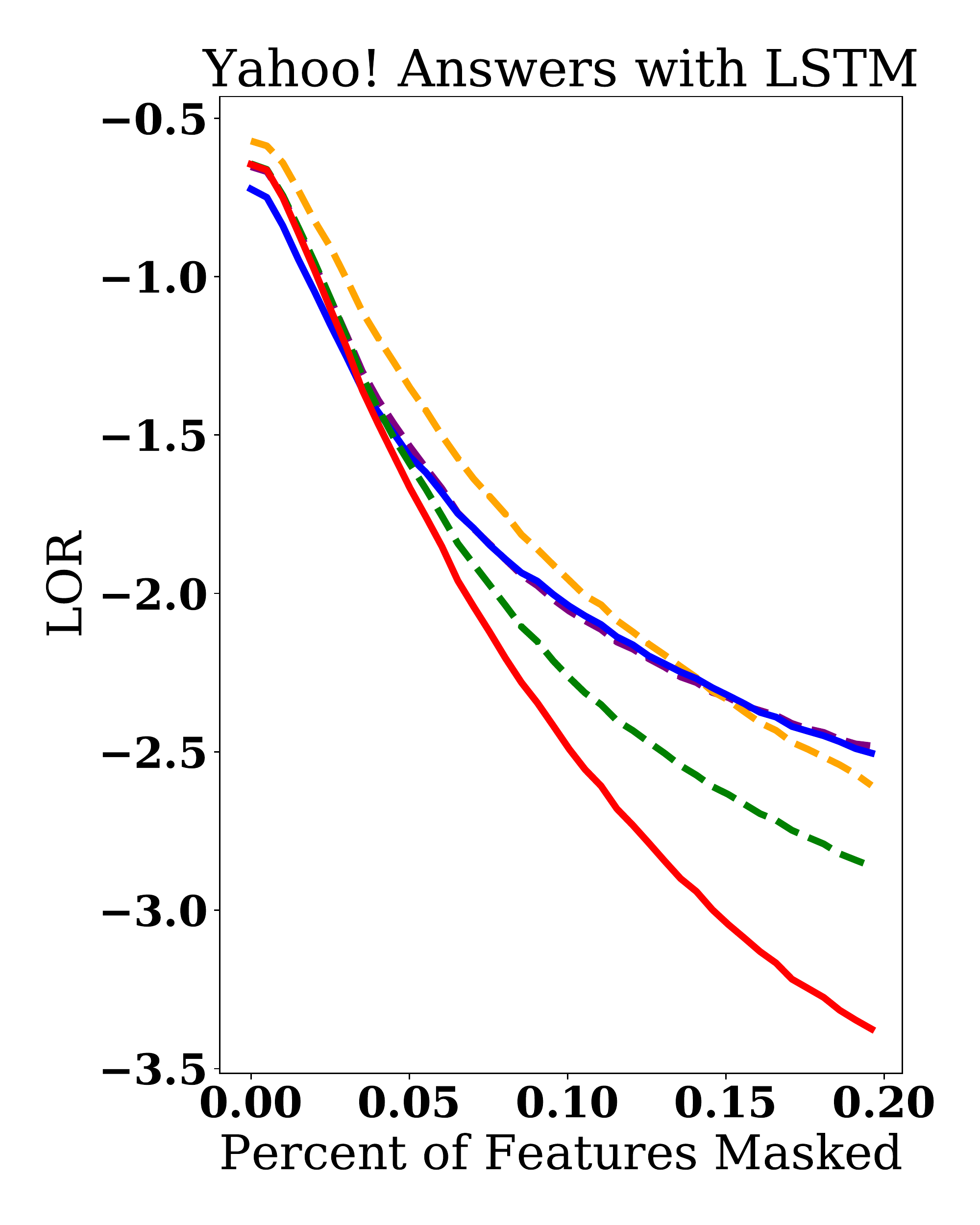} \\
\includegraphics[width=0.90\linewidth]{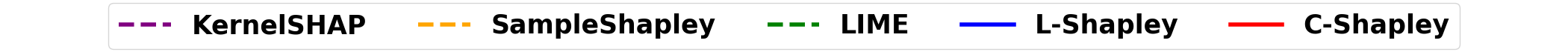}
\caption{The above plots show the change in log odds ratio of the
  predicted class as a function of the percent of masked features, on
  the three text data sets. Lower log odds ratios are better.}
\label{fig:text_lor_mask} 
\end{figure} 

\begin{table*}[bt!]
\centering
\resizebox{.67\textwidth}{!}{
 \begin{tabular}{||c | c||} 
 \hline
 Method & Explanation \\ [0.5ex] 
\hline\hline
Shapley& \colorbox[rgb]{1,0.92,0.92}{It} \colorbox[rgb]{0.97,0.97,1}{is} \colorbox[rgb]{0.3,0.3,1}{not} \colorbox[rgb]{1,0.65,0.65}{heartwarming} \colorbox[rgb]{0.8,0.8,1}{or} \colorbox[rgb]{1,0.5,0.5}{entertaining}. \colorbox[rgb]{1,0.81,0.81}{It} \colorbox[rgb]{0.77,0.77,1}{just} \colorbox[rgb]{0.0,0.0,1}{sucks}.\\ 
\hline
C-Shapley& \colorbox[rgb]{1,0.82,0.82}{It} \colorbox[rgb]{1,0.91,0.91}{is} \colorbox[rgb]{0.57,0.57,1}{not} \colorbox[rgb]{1,0.49,0.49}{heartwarming} \colorbox[rgb]{0.79,0.79,1}{or} \colorbox[rgb]{1,0.0,0.0}{entertaining}. \colorbox[rgb]{1,0.73,0.73}{It} \colorbox[rgb]{0.92,0.92,1}{just} \colorbox[rgb]{0.67,0.67,1}{sucks}.\\ 
\hline
L-Shapley& \colorbox[rgb]{1,0.89,0.89}{It} \colorbox[rgb]{1,0.93,0.93}{is} \colorbox[rgb]{0.26,0.26,1}{not} \colorbox[rgb]{1,0.72,0.72}{heartwarming} \colorbox[rgb]{0.36,0.36,1}{or} \colorbox[rgb]{1,0.0,0.0}{entertaining}. \colorbox[rgb]{1,0.39,0.39}{It} \colorbox[rgb]{0.96,0.96,1}{just} \colorbox[rgb]{0.65,0.65,1}{sucks}.\\ 
\hline
KernelSHAP& \colorbox[rgb]{1,0.71,0.71}{It} \colorbox[rgb]{0.95,0.95,1}{is} \colorbox[rgb]{0.48,0.48,1}{not} \colorbox[rgb]{1,1.0,1.0}{heartwarming} \colorbox[rgb]{0.78,0.78,1}{or} \colorbox[rgb]{1,0.15,0.15}{entertaining}. \colorbox[rgb]{1,1.0,1.0}{It} \colorbox[rgb]{1,1.0,1.0}{just} \colorbox[rgb]{0.0,0.0,1}{sucks}.\\ 
\hline
SampleShapley& \colorbox[rgb]{1,1.0,1.0}{It} \colorbox[rgb]{0.8,0.8,1}{is} \colorbox[rgb]{0.8,0.8,1}{not} \colorbox[rgb]{1,0.38,0.38}{heartwarming} \colorbox[rgb]{0.87,0.87,1}{or} \colorbox[rgb]{1,0.95,0.95}{entertaining}. \colorbox[rgb]{1,1.0,1.0}{It} \colorbox[rgb]{0.96,0.96,1}{just} \colorbox[rgb]{0.0,0.0,1}{sucks}.\\ 
\hline
\hline
 \end{tabular} 
 } 
 \caption{Each word is highlighted with the RGB color as a linear
   function of its importance score. The background colors of words
   with positive and negative scores are linearly interpolated between
   blue and white, red and white respectively.}
\label{tab:imdb_word_demo}
\end{table*}

\begin{figure}[!bt]
\centering
  \begin{tabular}[b]{ccc}
    \begin{subfigure}[b]{0.385\columnwidth}
      \hspace*{-0.0cm}\includegraphics[width=\textwidth]{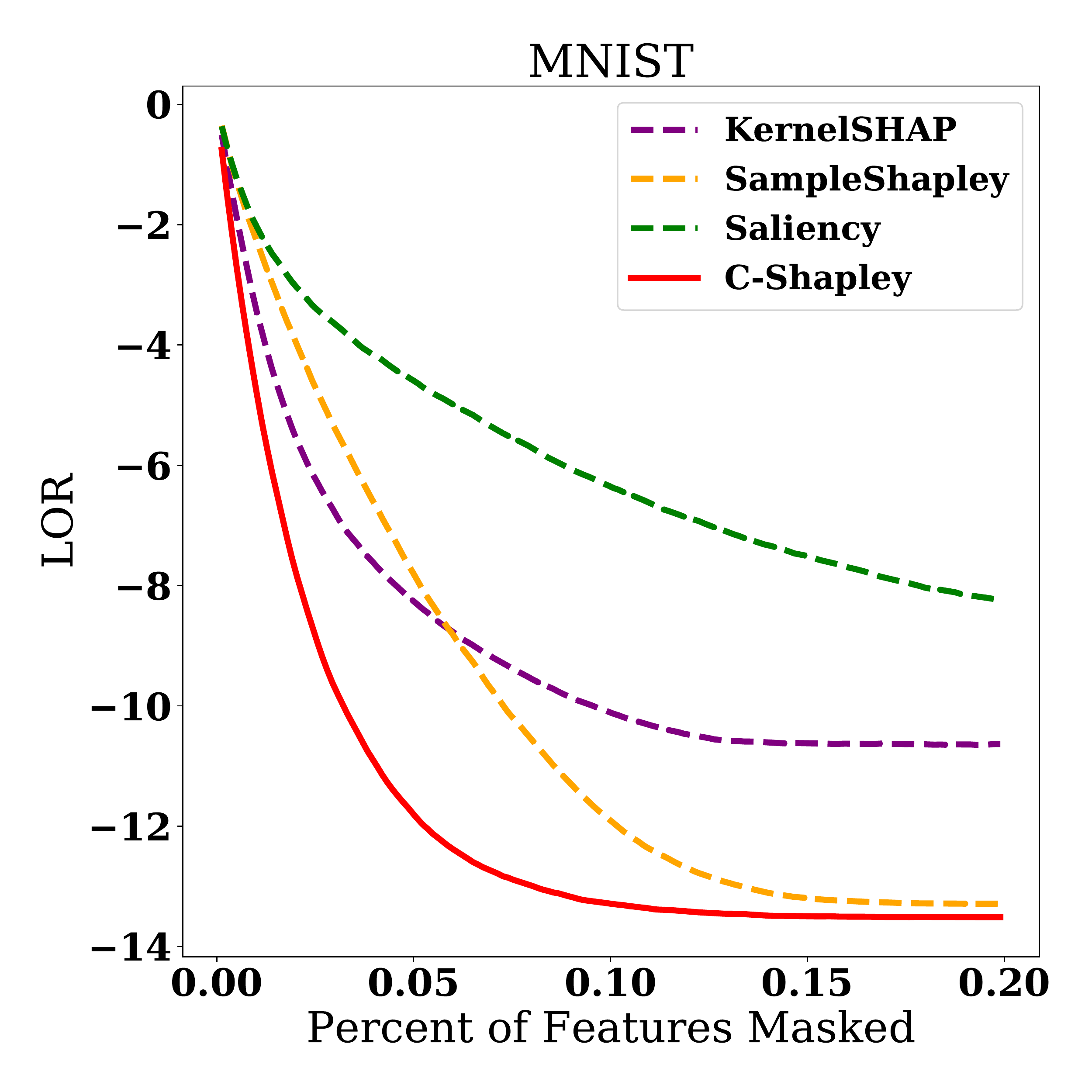}
    \end{subfigure}
    &
    \begin{subfigure}[b]{0.385\columnwidth}
      \hspace*{-0.4cm}\includegraphics[width=\textwidth]{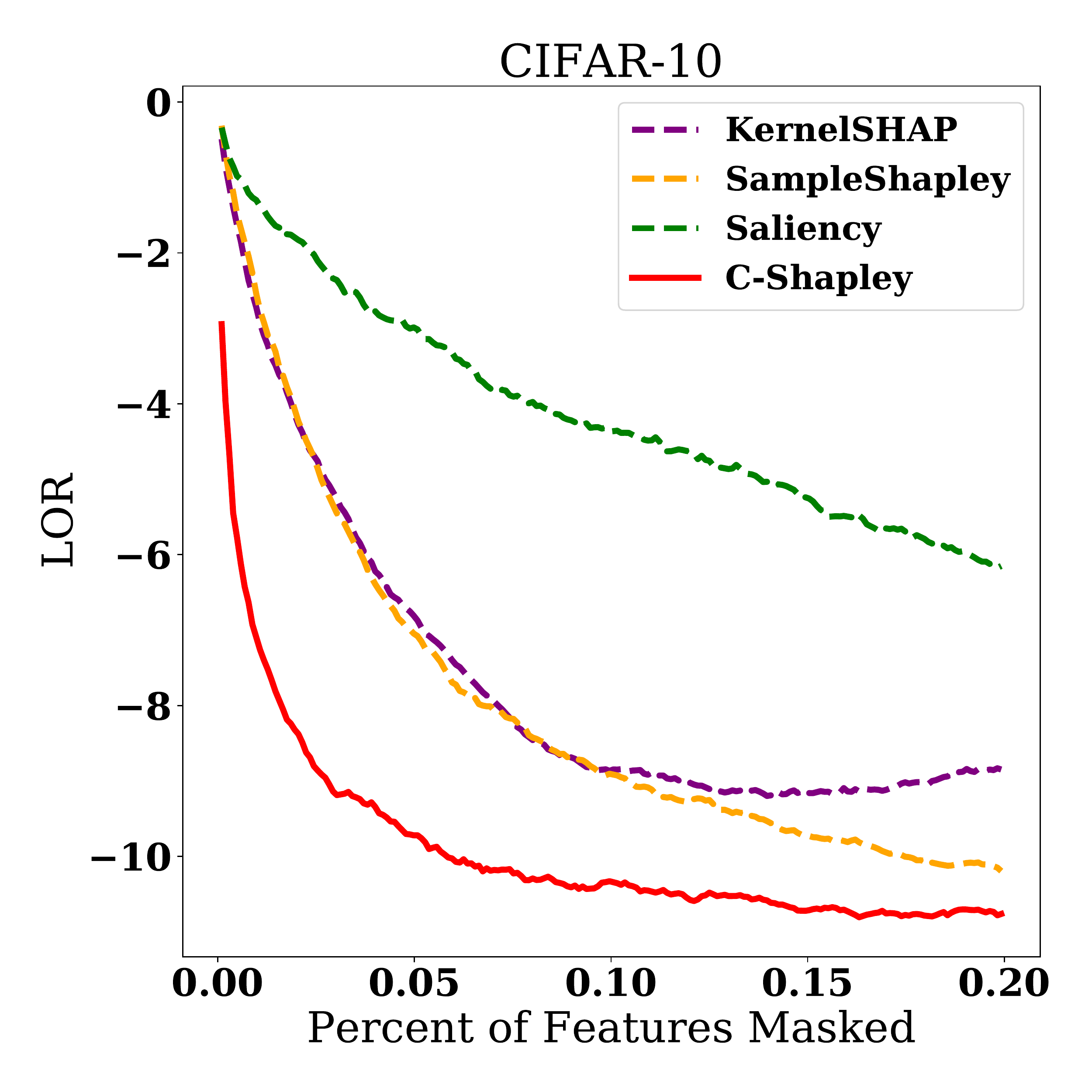}
      % \caption{C.}
      % % \label{fig:C}
    \end{subfigure}
    &
    \begin{tabular}[b]{c}
      \begin{subfigure}[b]{0.153\columnwidth}
        \hspace*{-0.5cm}\includegraphics[width=\textwidth]{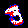}
        % \caption{A.}
        % \label{fig:A}
      \end{subfigure}\\
      \begin{subfigure}[b]{0.153\columnwidth}
        \hspace*{-0.5cm}\raisebox{+.27\totalheight}{\includegraphics[width=\textwidth]{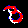}}
        % \caption{B.}
        % \label{fig:B}
      \end{subfigure}
    \end{tabular}  
  \end{tabular} 
  \caption{Left and Middle: change in log-odds ratio vs. the percent
    of pixels masked on MNIST and CIFAR10. Right: top pixels ranked by
    C-Shapley for a ``$3$'' and an ``$8$'' misclassified into ``$8$''
    and ``$3$'' respectively. The masked pixels are colored with red
    if activated (white) and blue otherwise.}
  \label{fig:mix}
\end{figure}

On IMDB with Word-CNN, the simplest model among the three, L-Shapley,
achieves the best performance while LIME, KernelSHAP and C-Shapley
achieve slightly worse performance. On AG's news with Char-CNN,
L-Shapley and C-Shapley both outperform other algorithms. On
Yahoo!\ Answers with LSTM, C-Shapley outperforms the rest of the
algorithms by a large margin, followed by LIME. L-Shapley with order
$1$, SampleShapley, and KernelSHAP do not perform well for LSTM model,
probably because some of the signals captured by LSTM are relatively
long $n$-grams.

We also visualize the importance scores produced by different
Shapley-based methods on Example~\eqref{example}, which is part of a
negative movie review taken from IMDB. The result is shown in
Table~\ref{tab:imdb_word_demo}. More visualizations by our methods are
available online.\footnoteref{ft:code}

%%%%%%%%%%%%%%%%%%%%%%%%%%%%%%%%%%%%%%%%%%%%%%%%%%%%%%%%%%%%%%%%%%%%%%%%%%%%%%%%

\subsection{Image Classification}

We carry out experiments in image classification on the MNIST and CIFAR10 data sets:

\begin{itemize}
\item \textbf{MNIST}: The MNIST data set contains $28\times 28$ images
  of handwritten digits with ten categories $0-9$
  \cite{lecun1998gradient}. A subset of MNIST data set composed of
  digits $3$ and $8$ is used for better visualization, with $12,000$
  images for training and $1,000$ images for testing. A simple CNN
  model achieves $99.7\%$ accuracy on the test data set.

\item \textbf{CIFAR10}: The CIFAR10 data set
  \cite{krizhevsky2009learning} contains $32\times 32$ images in ten
  classes. A subset of CIFAR10 data set composed of deers and horses
  is used for better visualization, with $10,000$ images for training
  and $2,000$ images for testing. A convolutional neural network
  modified from AlexNet \cite{krizhevsky2012imagenet} achieves
  $96.1\%$ accuracy on the test data set.
\end{itemize}

We take each pixel as a single feature for both MNIST and CIFAR10. We
choose the average pixel strength as the reference point for all
methods, and make $4\times d$ model evaluations, where $d$ is the
number of pixels for each input image, which keeps the number of model
evaluations under $4,000$.

LIME and L-Shapley are not used for comparison because LIME takes
``superpixels'' instead of raw pixels segmented by segmentation
algorithms as single features, and L-Shapley requires nearly sixteen
thousand model evaluations when applied to raw
pixels.\footnote{L-Shapley becomes practical if we take small patches
  of images instead of pixels as single features.} For C-Shapley, the
budget allows the regression-based version to evaluate all $n\times n$
image patches with $n\leq 4$.

Figure~\ref{fig:mix} shows the decrease in log-odds scores before and
after masking the top pixels ranked by importance scores as the
percentage of masked pixels over the total number of pixels increases
on $1,000$ test samples on MNIST and CIFAR10 data sets. C-Shapley
consistently outperforms other methods on both data sets.

Figure~\ref{fig:mnist} and Figure~\ref{fig:horse} provide additional
visualization of the results. By masking the top pixels ranked by
various methods, we find that the pixels picked by C-Shapley
concentrate around and inside the digits in MNIST. The C-Shapley and
Saliency methods yield the most interpretable results in CIFAR10. In
particular, C-Shapley tends to mask the parts of head and body that
distinguish deers and horses, and the human riding the horse.
Figure~\ref{fig:mix} shows two misclassified digits by the CNN
model. Interestingly, the top pixels chosen by C-Shapley visualize the
``reasoning'' of the model: more specifically, the important pixels to
the model are exactly those which could form a digit from the opposite
class.

\begin{figure}[!bt] 
\centering
\includegraphics[width=0.09\linewidth]{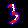}%
\includegraphics[width=0.09\linewidth]{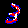}%
\includegraphics[width=0.09\linewidth]{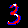}%
\includegraphics[width=0.09\linewidth]{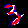}%
\includegraphics[width=0.09\linewidth]{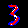}%
\includegraphics[width=0.09\linewidth]{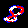}%
\includegraphics[width=0.09\linewidth]{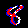}%
\includegraphics[width=0.09\linewidth]{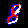}%
\includegraphics[width=0.09\linewidth]{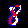}%
\includegraphics[width=0.09\linewidth]{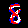}%
 
\includegraphics[width=0.09\linewidth]{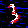}%
\includegraphics[width=0.09\linewidth]{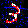}%
\includegraphics[width=0.09\linewidth]{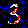}%
\includegraphics[width=0.09\linewidth]{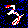}%
\includegraphics[width=0.09\linewidth]{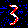}%
\includegraphics[width=0.09\linewidth]{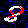}% 
\includegraphics[width=0.09\linewidth]{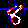}%
\includegraphics[width=0.09\linewidth]{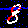}% 
\includegraphics[width=0.09\linewidth]{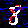}% 
\includegraphics[width=0.09\linewidth]{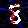}% 

\includegraphics[width=0.09\linewidth]{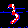}%
\includegraphics[width=0.09\linewidth]{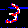}%
\includegraphics[width=0.09\linewidth]{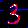}%
\includegraphics[width=0.09\linewidth]{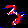}%
\includegraphics[width=0.09\linewidth]{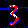}%
\includegraphics[width=0.09\linewidth]{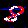}% 
\includegraphics[width=0.09\linewidth]{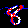}%
\includegraphics[width=0.09\linewidth]{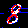}% 
\includegraphics[width=0.09\linewidth]{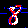}% 
\includegraphics[width=0.09\linewidth]{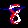}% 

\includegraphics[width=0.09\linewidth]{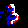}%
\includegraphics[width=0.09\linewidth]{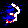}%
\includegraphics[width=0.09\linewidth]{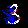}%
\includegraphics[width=0.09\linewidth]{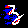}%
\includegraphics[width=0.09\linewidth]{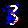}%
\includegraphics[width=0.09\linewidth]{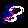}% 
\includegraphics[width=0.09\linewidth]{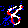}%
\includegraphics[width=0.09\linewidth]{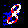}% 
\includegraphics[width=0.09\linewidth]{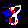}% 
\includegraphics[width=0.09\linewidth]{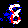}% 
\caption{ Some examples of explanations obtained for the MNIST data
  set. The first row corresponds to the original images, with the rows
  below showing images masked based on scores produced by C-Shapley,
  KernelSHAP, SampleShapley and Saliency respectively. For best
  visualization results, 15\% and 20\% of the pixels are masked for
  each image. The masked pixels are colored with red if activated
  (white) and blue otherwise.  }
\label{fig:mnist} 
\end{figure}

\begin{figure}[!bt]
\centering   

\includegraphics[width=0.09\linewidth]{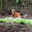}%
\includegraphics[width=0.09\linewidth]{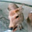}%
\includegraphics[width=0.09\linewidth]{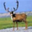}%
\includegraphics[width=0.09\linewidth]{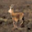}% 
\includegraphics[width=0.09\linewidth]{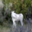}% 
\includegraphics[width=0.09\linewidth]{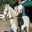}%   
\includegraphics[width=0.09\linewidth]{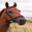}%   
\includegraphics[width=0.09\linewidth]{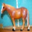}%   
\includegraphics[width=0.09\linewidth]{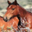}%   
\includegraphics[width=0.09\linewidth]{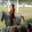}% 

\includegraphics[width=0.09\linewidth]{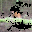}%
\includegraphics[width=0.09\linewidth]{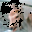}%
\includegraphics[width=0.09\linewidth]{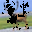}%
\includegraphics[width=0.09\linewidth]{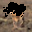}% 
\includegraphics[width=0.09\linewidth]{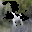}% 
\includegraphics[width=0.09\linewidth]{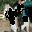}%   
\includegraphics[width=0.09\linewidth]{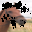}%   
\includegraphics[width=0.09\linewidth]{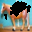}%   
\includegraphics[width=0.09\linewidth]{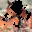}%   
\includegraphics[width=0.09\linewidth]{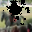}% 

\includegraphics[width=0.09\linewidth]{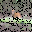}%
\includegraphics[width=0.09\linewidth]{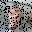}%
\includegraphics[width=0.09\linewidth]{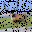}%
\includegraphics[width=0.09\linewidth]{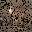}% 
\includegraphics[width=0.09\linewidth]{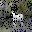}% 
\includegraphics[width=0.09\linewidth]{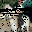}%   
\includegraphics[width=0.09\linewidth]{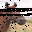}%   
\includegraphics[width=0.09\linewidth]{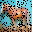}%   
\includegraphics[width=0.09\linewidth]{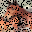}%   
\includegraphics[width=0.09\linewidth]{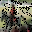}% 

\includegraphics[width=0.09\linewidth]{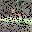}%
\includegraphics[width=0.09\linewidth]{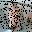}%
\includegraphics[width=0.09\linewidth]{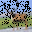}%
\includegraphics[width=0.09\linewidth]{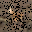}% 
\includegraphics[width=0.09\linewidth]{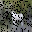}% 
\includegraphics[width=0.09\linewidth]{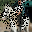}%   
\includegraphics[width=0.09\linewidth]{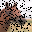}%   
\includegraphics[width=0.09\linewidth]{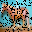}%   
\includegraphics[width=0.09\linewidth]{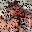}%   
\includegraphics[width=0.09\linewidth]{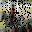}% 

\includegraphics[width=0.09\linewidth]{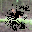}%
\includegraphics[width=0.09\linewidth]{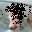}%
\includegraphics[width=0.09\linewidth]{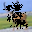}%
\includegraphics[width=0.09\linewidth]{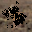}% 
\includegraphics[width=0.09\linewidth]{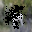}%   
\includegraphics[width=0.09\linewidth]{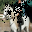}%   
\includegraphics[width=0.09\linewidth]{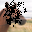}%   
\includegraphics[width=0.09\linewidth]{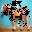}%   
\includegraphics[width=0.09\linewidth]{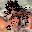}%   
\includegraphics[width=0.09\linewidth]{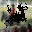}% 

\caption{Some examples of explanations obtained for the CIFAR10 data
  set. The first row corresponds to the original images, with the rows
  below showing images masked based on scores produced by C-Shapley,
  KernelSHAP, SampleShapley and Saliency respectively. For best
  visualization results, 20\% of the pixels are masked for each
  image.}
\label{fig:horse} 
\end{figure}  

%%%%%%%%%%%%%%%%%%%%%%%%%%%%%%%%%%%%%%%%%%%%%%%%%%%%%%%%%%%%%%%%%%%%%%%%%%%%%%%%

\section{Discussion}
\label{SecDiscussion}

We have proposed L-Shapley and C-Shapley for instancewise feature importance scoring, making use of a graphical representation of the data. We have shown the superior performance of the proposed algorithms compared to other methods for instancewise feature importance scoring in text and image classification.

\section*{Acknowledgments}

We would like to acknowledge support from the DARPA Program on
Lifelong Learning Machines from the Army Research Office under grant
number W911NF-17-1-0304, and from National Science Foundation grant
NSF-DMS-1612948.

\vspace*{.25cm}
\bibliography{explanation}

\begin{thebibliography}{25}
\providecommand{\natexlab}[1]{#1}
\providecommand{\url}[1]{\texttt{#1}}
\expandafter\ifx\csname urlstyle\endcsname\relax
  \providecommand{\doi}[1]{doi: #1}\else
  \providecommand{\doi}{doi: \begingroup \urlstyle{rm}\Url}\fi

\bibitem[Bach et~al.(2015)Bach, Binder, Montavon, Klauschen, M{\"u}ller, and
  Samek]{bach2015pixel}
Sebastian Bach, Alexander Binder, Gr{\'e}goire Montavon, Frederick Klauschen,
  Klaus-Robert M{\"u}ller, and Wojciech Samek.
\newblock On pixel-wise explanations for non-linear classifier decisions by
  layer-wise relevance propagation.
\newblock \emph{PloS One}, 10\penalty0 (7):\penalty0 e0130140, 2015.

\bibitem[Baehrens et~al.(2010)Baehrens, Schroeter, Harmeling, Kawanabe, Hansen,
  and M{\"u}ller]{baehrens2010explain}
David Baehrens, Timon Schroeter, Stefan Harmeling, Motoaki Kawanabe, Katja
  Hansen, and Klaus-Robert M{\"u}ller.
\newblock How to explain individual classification decisions.
\newblock \emph{Journal of Machine Learning Research}, 11:\penalty0 1803--1831,
  2010.

\bibitem[Chen et~al.(2018)Chen, Song, Wainwright, and Jordan]{chen2018learning}
Jianbo Chen, Le~Song, Martin~J Wainwright, and Michael~I Jordan.
\newblock Learning to explain: An information-theoretic perspective on model
  interpretation.
\newblock \emph{arXiv preprint arXiv:1802.07814}, 2018.

\bibitem[Cover and Thomas(2012)]{cover2012elements}
Thomas~M Cover and Joy~A Thomas.
\newblock \emph{Elements of Information Theory}.
\newblock John Wiley \& Sons, 2012.

\bibitem[Datta et~al.(2016)Datta, Sen, and Zick]{datta2016algorithmic}
Anupam Datta, Shayak Sen, and Yair Zick.
\newblock Algorithmic transparency via quantitative input influence: Theory and
  experiments with learning systems.
\newblock In \emph{Security and Privacy (SP), 2016 IEEE Symposium on}, pages
  598--617. IEEE, 2016.

\bibitem[Hinton et~al.()Hinton, Srivastava, and Swersky]{hinton2012neural}
Geoffrey Hinton, Nitish Srivastava, and Kevin Swersky.
\newblock Neural networks for machine learning-lecture 6a-overview of
  mini-batch gradient descent.

\bibitem[Hochreiter and Schmidhuber(1997)]{hochreiter1997long}
Sepp Hochreiter and J{\"u}rgen Schmidhuber.
\newblock Long short-term memory.
\newblock \emph{Neural Computation}, 9\penalty0 (8):\penalty0 1735--1780, 1997.

\bibitem[Kim(2014)]{kim2014convolutional}
Yoon Kim.
\newblock Convolutional neural networks for sentence classification.
\newblock In \emph{Proceedings of the 2014 Conference on Empirical Methods in
  Natural Language Processing (EMNLP)}, pages 1746--1751, 2014.

\bibitem[Kingma and Ba(2015)]{kingma2014adam}
Diederik~P Kingma and Jimmy Ba.
\newblock Adam: A method for stochastic optimization.
\newblock In \emph{Proceedings of the International Conference on Learning
  Representations (ICLR)}, 2015.

\bibitem[Krizhevsky(2009)]{krizhevsky2009learning}
Alex Krizhevsky.
\newblock Learning multiple layers of features from tiny images.
\newblock 2009.

\bibitem[Krizhevsky et~al.(2012)Krizhevsky, Sutskever, and
  Hinton]{krizhevsky2012imagenet}
Alex Krizhevsky, Ilya Sutskever, and Geoffrey~E Hinton.
\newblock Imagenet classification with deep convolutional neural networks.
\newblock In \emph{Advances in Neural Information Processing Systems}, pages
  1097--1105, 2012.

\bibitem[LeCun et~al.(1998)LeCun, Bottou, Bengio, and
  Haffner]{lecun1998gradient}
Yann LeCun, L{\'e}on Bottou, Yoshua Bengio, and Patrick Haffner.
\newblock Gradient-based learning applied to document recognition.
\newblock \emph{Proceedings of the IEEE}, 86\penalty0 (11):\penalty0
  2278--2324, 1998.

\bibitem[Lipton(2016)]{lipton2016mythos}
Zachary~C Lipton.
\newblock The mythos of model interpretability.
\newblock \emph{arXiv preprint arXiv:1606.03490}, 2016.

\bibitem[Lundberg and Lee(2017)]{lundberg2017unified}
Scott~M Lundberg and Su-In Lee.
\newblock A unified approach to interpreting model predictions.
\newblock In I.~Guyon, U.~V. Luxburg, S.~Bengio, H.~Wallach, R.~Fergus,
  S.~Vishwanathan, and R.~Garnett, editors, \emph{Advances in Neural
  Information Processing Systems 30}, pages 4765--4774. Curran Associates,
  Inc., 2017.
\newblock URL
  \url{http://papers.nips.cc/paper/7062-a-unified-approach-to-interpreting-model-predictions.pdf}.

\bibitem[Maas et~al.(2011)Maas, Daly, Pham, Huang, Ng, and
  Potts]{maas2011learning}
Andrew~L Maas, Raymond~E Daly, Peter~T Pham, Dan Huang, Andrew~Y Ng, and
  Christopher Potts.
\newblock Learning word vectors for sentiment analysis.
\newblock In \emph{Proceedings of the 49th Annual Meeting of the Association
  for Computational Linguistics}, pages 142--150. Association for Computational
  Linguistics, 2011.

\bibitem[Myerson(1977)]{myerson1977graphs}
Roger~B Myerson.
\newblock Graphs and cooperation in games.
\newblock \emph{Mathematics of Operations Research}, 2\penalty0 (3):\penalty0
  225--229, 1977.

\bibitem[Ribeiro et~al.(2016)Ribeiro, Singh, and Guestrin]{ribeiro2016should}
Marco~Tulio Ribeiro, Sameer Singh, and Carlos Guestrin.
\newblock Why should {I} trust you?: Explaining the predictions of any
  classifier.
\newblock In \emph{Proceedings of the 22nd ACM SIGKDD International Conference
  on Knowledge Discovery and Data Mining}, pages 1135--1144. ACM, 2016.

\bibitem[Shapley(1953)]{shapley1953value}
Lloyd~S Shapley.
\newblock A value for n-person games.
\newblock \emph{Contributions to the Theory of Games}, 2\penalty0
  (28):\penalty0 307--317, 1953.

\bibitem[Shrikumar et~al.(2017)Shrikumar, Greenside, and
  Kundaje]{shrikumar2016not}
Avanti Shrikumar, Peyton Greenside, and Anshul Kundaje.
\newblock Learning important features through propagating activation
  differences.
\newblock In \emph{ICML}, volume~70 of \emph{Proceedings of Machine Learning
  Research}, pages 3145--3153. PMLR, 06--11 Aug 2017.

\bibitem[Simonyan et~al.(2014)Simonyan, Vedaldi, and
  Zisserman]{simonyan2013deep}
K.~Simonyan, A.~Vedaldi, and A.~Zisserman.
\newblock Deep inside convolutional networks: Visualising image classification
  models and saliency maps.
\newblock In \emph{Proceedings of the International Conference on Learning
  Representations (ICLR)}, 2014.

\bibitem[Srivastava et~al.(2014)Srivastava, Hinton, Krizhevsky, Sutskever, and
  Salakhutdinov]{srivastava2014dropout}
Nitish Srivastava, Geoffrey Hinton, Alex Krizhevsky, Ilya Sutskever, and Ruslan
  Salakhutdinov.
\newblock Dropout: A simple way to prevent neural networks from overfitting.
\newblock \emph{The Journal of Machine Learning Research}, 15\penalty0
  (1):\penalty0 1929--1958, 2014.

\bibitem[{\v{S}}trumbelj and Kononenko(2010)]{vstrumbelj2010efficient}
Erik {\v{S}}trumbelj and Igor Kononenko.
\newblock An efficient explanation of individual classifications using game
  theory.
\newblock \emph{Journal of Machine Learning Research}, 11:\penalty0 1--18,
  2010.

\bibitem[Sundararajan et~al.(2017)Sundararajan, Taly, and
  Yan]{sundararajan2017axiomatic}
Mukund Sundararajan, Ankur Taly, and Qiqi Yan.
\newblock Axiomatic attribution for deep networks.
\newblock In \emph{International Conference on Machine Learning}, pages
  3319--3328, 2017.

\bibitem[Young(1985)]{young1985monotonic}
H~Peyton Young.
\newblock Monotonic solutions of cooperative games.
\newblock \emph{International Journal of Game Theory}, 14\penalty0
  (2):\penalty0 65--72, 1985.

\bibitem[Zhang et~al.(2015)Zhang, Zhao, and LeCun]{zhang2015character}
Xiang Zhang, Junbo Zhao, and Yann LeCun.
\newblock Character-level convolutional networks for text classification.
\newblock In \emph{Advances in Neural Information Processing Systems}, pages
  649--657, 2015.

\end{thebibliography}
\bibliographystyle{plainnat} 

%\newpage

\appendix 

\section{Model structure}
\label{app:model}
\paragraph{IMDB Review with Word-CNN}

The word-based CNN model is composed of a $50$-dimensional word
embedding, a $1$-D convolutional layer of 250 filters and kernel size
three, a max-pooling and a $250$-dimensional dense layer as hidden
layers. Both the convolutional and the dense layers are followed by
ReLU as nonlinearity, and Dropout~\cite{srivastava2014dropout} as
regularization. The model is trained with rmsprop
\cite{hinton2012neural}. The model achieves an accuracy of $90.1\%$ on
the test data set.

\paragraph{AG's news with Char-CNN} The character-based CNN has
the same structure as the one proposed in \citet{zhang2015character},
composed of six convolutional layers, three max-pooling layers, and
two dense layers. The model is trained with SGD with momentum 0.9 and
decreasing step size initialized at $0.01$. (Details can be found in
\citet{zhang2015character}.) The model reaches accuracy of $90.09\%$
on the test data set.

\paragraph{Yahoo!\ Answers with LSTM}

The network consists of a $300$-dimensional randomly-initialized word
embedding, a bidirectional LSTM, each LSTM unit of dimension $256$,
and a dropout layer as hidden layers. The model is trained with
rmsprop \cite{hinton2012neural}. The model reaches accuracy of
$70.84\%$ on the test data set, close to the state-of-the-art accuracy
of $71.2\%$ obtained by character-based CNN \cite{zhang2015character}.

\paragraph{MNIST} A simple CNN model is trained on the data set, which
achieves $99.7\%$ accuracy on the test data set. It is composed of two
convolutional layers of kernel size $5\times 5$ and a dense linear
layer at last. The two convolutional layers contain 8 and 16 filters
respectively, and both are followed by a max-pooling layer of pool
size two.

\paragraph{CIFAR10} A convolutional neural network modified from
AlexNet~\cite{krizhevsky2012imagenet} is trained on the subset. It is
composed of six convolutional layers of kernel size $3\times 3$ and
two dense linear layers of dimension 512 and 256 at last. The six
convolutional layers contain 48,48,96,96,192,192 filters respectively,
and every two convolutional layers are followed by a max-pooling layer
of pool size two and a dropout layer. The CNN model is trained with
the Adam optimizer~\cite{kingma2014adam} and achieves $96.1\%$
accuracy on the test data set.

%%%%%%%%%%%%%%%%%%%%%%%%%%%%%%%%%%%%%%%%%%%%%%%%%%%%%%%%%%%%%%%%%%%%%%%%%%%%%%%%%%

\section{Proof of Theorems}
\label{app:proof}

In this appendix, we collect the proofs of Theorems 1 and 2.

\subsection{Proof of Theorem 1}

We state an elementary combinatorial equality required for the proof
of the main theorem:
\begin{lemma}[A combinatorial equality] For any positive
  integer $n$, and any pair of non-negative integers with $s \geq t$,
  we have
\begin{align}
\sum_{j=0}^{n} \frac {1}{\binom {n+s}{j+t}} \binom
    {n}{j}=\frac{s+1+n}{(s+1)\binom{s}{t}} %\frac{\binom{s+1+n}{t+n}}{\binom{s+n}{t+n}\binom{s+1}{t}}.
\end{align}
\label{lemma:comb}
\end{lemma}
\begin{proof} 
By the binomial theorem for negative integer exponents, we have
\begin{align*}
\frac{1}{(1-x)^{t+1}} = \sum_{j=0}^\infty \binom{j+t}{j}x^j.
\end{align*}
The identity can be found by examination of the coefficient of $x^n$ in the expansion of 
\begin{equation}\label{eq:nbt}
\frac{1}{(1-x)^{t+1}} \cdot \frac{1}{(1-x)^{s-t+1}} = \frac{1}{(1-x)^{s+1+1}}.
\end{equation}
In fact, equating the coefficients of $x^n$ in the left and the right hand sides, we get
\begin{equation}
\sum_{j=0}^n\binom{j+t}{j}\binom{(n-j)+(s-t)}{n-j} = \binom{n+s+1}{n} = \frac{n+s+1}{s+1}\binom{n+s}{n}.
\end{equation} 
Moving $\binom{n+s}{n}$ to the right hand side and expanding the binomial coefficients, we have
\begin{equation}
\sum_{j=0}^n\frac{(j+t)!}{j!t!}\cdot \frac{(n-j+s-t)!}{(n-j)!(s-t)!}\cdot \frac{n!s!}{(n+s)!}=\frac{n+s+1}{s+1},
\end{equation}
which implies
\begin{align*}
\sum_{j=0}^n \binom{n}{j}\binom{s}{t}\bigg /\binom{n+s}{j+t} &=\sum_{j=0}^n \frac{n!}{(n-j)!j!}\cdot\frac{s!}{t!(s-t)!}\cdot\frac{((n+s)-(j+t))!(j+t)!}{(n+s)!}\\
&= \sum_{j=0}^n\frac{(j+t)!}{j!t!}\cdot \frac{(n-j+s-t)!}{(n-j)!(s-t)!}\cdot \frac{n!s!}{(n+s)!}=\frac{n+s+1}{s+1}.
\end{align*} 
\end{proof}
% \mjwcomment{Need to cite a reference to the proof of this fact.}

Taking this lemma, we now prove the theorem. We split our analysis into two cases, namely  $S = \nbhd{k}{i}$  versus $S \subset \nbhd{k}{i}$. For notational convenience, we extend the definition of L-Shapley estimate for feature $i$ to an arbitrary feature subset $S$ containing $i$. In particular, we define 
\begin{align}
\hat\phi_x^S(i) & \defn \frac{1}{|S|} \sum_{ \substack{T \ni
    i \\T \subseteq S }}
\frac{1}{\binom{|S|-1}{|T|-1}} m_x(T,i).
\end{align}

\paragraph{Case 1:} First, suppose that
$S = \nbhd{k}{i}$. For any subset $A\subset [d]$, we introduce the shorthand
notation $U_S(A) \defn A \cap S$ and $V_S(A) \defn A\cap S^c$, and
note that $A = U_S(A) \cup V_S(A)$.  Recalling the definition of the
Shapley value, let us partition all the subsets $A$ based on $U_S(A)$,
in particular writing
\begin{align*}
\phi_X(i) & = \frac{1}{d} \sum_{ \substack{ A \subseteq [d] \\ A \ni
    i}} \frac{1}{\binom{d-1}{|A|-1}} m_X(A,i) \; = \; \frac{1}{d}
\sum_{\substack{U \subseteq S \\ U \ni i}} \sum_{ \substack{ A
    \subseteq [d] \\ U_S(A) = U}} \frac {1}{\binom{d-1}{|A|-1} }
m_X(A,i).
\end{align*}
Based on this partitioning, the expected error between $\hat
\phi_X^{S}(i)$ and $\phi_X(i)$ can be written as
\begin{align}
\label{eq:expected_local}  
\Exs \left| \hat \phi_X^{S}(i) - \phi_X(i) \right | & = \Exs \left |
\frac{1}{|S|} \sum_{ \substack{U \subseteq S \\ U \ni i}}
\frac{1}{\binom{|S|-1}{|U|-1}} m_X(U,i) - \frac{1}{d} \sum_{
  \substack{U \subseteq S \\ U \ni i }} \sum_{\substack{A \subseteq
    [d] \\ U_S(A) = U}} \frac {1}{\binom{d-1}{|A|-1}} m_X(A,i)
\right|.
\end{align}
Partitioning the set $\{A : U_S(A) = U\}$ by the size of $V_S(A) = A
\cap S^c$, we observe that
\begin{align*}
\sum_{ \substack{A \subseteq [d] \\ U_S(A) = U}}
\frac{1}{\binom{d-1}{|A|-1}} & = \sum_{i=0}^{d-|S|}
\frac{1}{\binom{d-1}{i+|U|-1}} \binom{d-|S|}{i} \\
%
% & =
% \frac{\binom{d-|S|+1+|S|+1}{d-|S|+|U|-1}}{\binom{|S|-1+d-|S|}{d-|S|+|U|-1}
%   \binom{|S|-1+1}{|U|-1}} \\
% %
% & =\frac{d}{\binom{|S|}{|U|-1}(|S|-|U|+1)} \\
%
& = \frac{(|S|-1) + 1 + (d-|S|)}{((|S|-1)+1)\binom{|S|-1}{|U|-1}}\\
& = \frac{d}{|S|} \frac{1}{\binom{|S|-1}{|U|-1}},
\end{align*}
where we have applied Lemma~\ref{lemma:comb} with $n = d - |S|$, $s =
|S| - 1$, and $t = |U| - 1$.  Substituting this equivalence into
equation~\eqref{eq:expected_local}, we find that the expected error
can be upper bounded by
\begin{align}
  \label{EqnEspresso}
  \Exs |\hat \phi_X^S(i) - \phi_X(i)| & \leq \frac{1}{d} \sum_{
    \substack{ U \subseteq S \\ U \ni i}} \sum_{ \substack{ A
      \subseteq [d] \\ U_S(A) = U}} \frac{1}{\binom{d-1}{|A|-1}} \Exs
  \left| m_X(U,i) - m_X(A,i) \right|,
\end{align}
where we recall that $A = U_S(A) \cup V_S(A)$.

Now omitting the dependence of $U_S(A),V_S(A)$ on $A$ for notational
simplicity, we now write the difference as
\begin{align*}
m_X(A,i) - m_X(U,i) & = \ExpModel \left [ \log \frac{\ProbModel
    (Y|X_{U\cup V})}{\ProbModel (Y|X_{U\cup V\setminus\{i\}})} - \log
  \frac {\ProbModel(Y|X_{U})}{\ProbModel(Y|X_{U\setminus\{i\}})} \mid
  X \right] \\
& = \ExpModel \left [ \log \frac {\Prob (Y,X_{U\setminus\{i\}})\Prob
    (X_U)P(X_{U\cup V\setminus\{i\}})P(X_{U\cup V}, Y)} {\Prob
    (Y,X_U)\Prob (X_{U\setminus\{i\}})P(X_{U\cup V})P(X_{U\cup
      V\setminus\{i\}}, Y)}\mid X \right] \\
& = \ExpModel \left[ \log \frac{\Prob(X_i,X_V \mid
    X_{U\setminus\{i\}},Y)}{\Prob(X_i \mid X_{U\setminus\{i\}},Y)
    \Prob(X_V \mid X_{U\setminus\{i\}},Y)} - \log \frac{\Prob
    (X_i,X_V|X_{U\setminus\{i\}})}{\Prob (X_i|X_{U\setminus\{i\}})
    \Prob (X_V \mid X_{U\setminus\{i\}})} \mid X \right].
\end{align*}
Substituting this equivalence into our earlier
bound~\eqref{EqnEspresso} and taking an expectation over $X$ on both
sides, we find that the expected error is upper bounded as
\begin{multline*}
  \Exs |\hat \phi_X^S(i) - \phi_X(i)| \leq \frac{1}{d} \sum_{
    \substack{ U \subseteq S \\ U \ni i}} \sum_{ \substack{A \subseteq
      [d] \\ U_S(A) = U}} \frac{1}{\binom{d-1}{|A|-1}} \Biggr \{ \Exs
  \left| \log \frac {\Prob(X_i,X_{V_S(A)}|X_{U\setminus
      \{i\}},Y)}{\Prob(X_i|X_{U\setminus
      \{i\}},Y)\Prob(X_{V_S(A)}|X_{U\setminus \{i\}},Y)}\right | \\
+ \Exs \left | \log \frac {\Prob (X_i,X_{V_S(A)}|X_{U\setminus
    \{i\}})}{\Prob (X_i \mid X_{U\setminus \{i\}})\Prob
  (X_{V_S(A)}|X_{U\setminus \{i\}})} \right| \Biggr \}.
\end{multline*}
Recalling the definition of the absolute mutual information, we see
that
\begin{align*}
  \Exs |\hat \phi_X^S(i) - \phi_X(i)| & \leq \frac{1}{d}
  \sum_{\substack{U \subseteq S \\ U \ni i}} \sum_{ \substack{A
      \subseteq [d] \\ U_S(A) = U}} \frac {1}{\binom{d-1}{|A|-1}} \Big
  \{ I_a(X_i; X_{V_S(A)} \mid X_{U\setminus \{i\}},Y) +
  I_a(X_i;X_{V_S(A)} \mid X_{U \setminus \{i\}}) \Big \} \\
& \leq 2 \varepsilon,
\end{align*}
which completes the proof of the claimed bound.

Finally, in the special case that $X_i\independent X_{[d]\setminus S}
| X_T$ and $X_i\independent X_{[d]\setminus S} | X_T,Y$ for any
$T\subset S$, then this inequality holds with $\varepsilon=0$, which
implies $\Exs |\hat \phi_X^S(i) - \phi_X(i)|=0$. Therefore, we have
$\hat \phi_X^S(i) = \phi_X(i)$ almost surely, as claimed.

%%%%%%%%%%%%%%%%%%%%%%%%%%%%%%%%%%%%%%%%%%%%%%%%%%%%%%%%%%%%%%%%%%%%%%%%%%%

\paragraph{Case 2:}  We now consider the general case in which $S\subset \nbhd{k}{i}$. Using the previous arguments, we
can show
\begin{align*}
\Exs |\hat \phi_X^{S}(i) - \phi_X^{k}(i)| \leq 2 \varepsilon, \quad
\mbox{and} \quad \Exs |\hat \phi_X^{S}(i) - \phi_X(i)| \leq 2
\varepsilon.
\end{align*}
Appylying the triangle inequality yields $\Exs |\hat \phi_X^{k}(i) -
\phi_X(i)| \leq 4 \varepsilon$, which establishes the claim.

%%%%%%%%%%%%%%%%%%%%%%%%%%%%%%%%%%%%%%%%%%%%%%%%%%%%%%%%%%%%%%%%%%%%%%%%%%%%%%%%%

\subsection{Proof of Theorem 2}

% \mjwcomment{Please rewrite this proof using the same notation and
%   style as how I rewrote the previous one.  Please pay particular
%   attention to how you define the $\sum$ terms ---- the way that you
%   wrote the proofs left it very unclear what is being summed over etc.
%   Also please replace $U$ to $U_S$ and $V$ to $V_S$.  Need to be clear
%   that these objects depend on $S$, otherwise many of your sums are
%   very hard to parse.}

As in the previous proof, we divide our analysis into two cases.

\paragraph{Case 1:}  First, 
suppose that $S = \nbhd{k}{i} = [d]$. For any subset $A\subset S$ with $i\in
A$, we can partition $A$ into two components $U_S(A)$ and $V_S(A)$, such
that $i\in U_S(A)$ and $U_S(A)$ is a connected subsequence. $V_S(A)$ is
disconnected from $U_S(A)$. We also define
\begin{align}
\mathcal C = \{U\mid i\in U, U\subset [d], U \text{ is a connected subsequence.}\}
\end{align}
We partition all the subsets $A\subset S$ based on $U_S(A)$ in the definition of the Shapley value:
\begin{align*}
\phi_X(i) &=\frac 1 d \sum_{\substack{ A \subseteq S \\ A \ni i}}\frac {1}{\binom{d-1}{|A|-1}}m_X(A,i)\\
&=\frac 1 d\sum_{U\in\mathcal C} \sum_{A:U_S(A) = U}\frac {1}{\binom{d-1}{|A|-1}}m_X(A,i).
\end{align*}
The expected error between $\tilde \phi_X^{[d]}(i)$ and $\phi_X(i)$ is
\begin{align}
\Exs |\tilde \phi_X^{[d]}(i) - \phi_X(i)| = \Exs \left |\frac 1
        {d} \sum_{U\in \mathcal C}\frac{2d}{(|U|+2)(|U|+1)|U|}
        m_X(U,i) - \frac 1 d\sum_{U\in\mathcal C} \sum_{A:U_S(A) =
          U}\frac
        {1}{\binom{d-1}{|A|-1}}m_X(A,i)\right |. \label{eq:expected1}
\end{align}
Partitioning $\{A:U_S(A) = U\}$ by the size of $V_S(A)$, we observe that
\begin{align*}
\sum_{A:U_S(A) = U}\frac {1}{\binom{d-1}{|A|-1}} &=
\sum_{i=0}^{d-|U|-2}\frac {1}{\binom{d-1}{i+|U|-1}}\binom{d-|U|-2}{i}\\ 
% &= \frac
%    {\binom{d-|U|-2+1+|U|+1}{d-|U|-2+|U|-1}}{\binom{d-|U|-2+|U|+1}{d-|U|-2+|U|-1}\binom{|U|+1+1}{|U|-1}}\\ 
%    &=\frac
%    {\binom{d}{d-3}}{\binom{d-1}{d-3}\binom{|U|+2}{|U|-1}} \\ 
%    &= \frac{d(d-1)(d-2)}{(d-1)(d-2)(|U|+2)(|U|+1)|U|/2}\\ 
&=\frac{(|U|+1)+1+(d-|U|-2)}{((|U|+1)+1)\binom{|U|+1}{|U|-1}}\\
   &=\frac{2d}{(|U|+2)(|U|+1)|U|},
\end{align*}
where we apply Lemma~\ref{lemma:comb} with $n=d-|U|-2$, $s=|U|+1$ and
$ t=|U|-1$. From equation~\eqref{eq:expected1}, the expected error can
be upper bounded by
\begin{align*}
\Exs \left| \tilde \phi_X^{[d]}(i) - \phi_X(i) \right| & \leq
\frac{1}{d} \sum_{U\in \mathcal C}\sum_{A:U_S(A) = U}
\frac{1}{\binom{d-1}{|A|-1}} \Exs \left| m_X(U,i) - m_X(A,i) \right|,
\end{align*}
where $A = U_S(A)\cup V_S(A)$. We omit the dependence of $U_S(A)$ and
$V_S(A)$ on the pair $(A,S)$ for notational simplicity, and observe
that the difference between $m_x(A,i)$ and $m_x(U,i)$ is
\begin{align*}
m_X(A,i) - m_X(U,i) & = \ExpModel \left [\log \frac{\ProbModel (Y|X_{U\cup
      V})}{\ProbModel (Y|X_{U \cup V\setminus\{i\}})} - \log \frac
  {\ProbModel(Y|X_{U})}{\ProbModel(Y|X_{U\setminus\{i\}})} \mid
  X\right ]\\
& = \ExpModel \left [\log \frac {\Prob (Y,X_{U\setminus\{i\}})\Prob
    (X_U)P(X_{U\cup V\setminus\{i\}})P(X_{U\cup V}, Y)} {\Prob
    (Y,X_U)\Prob (X_{U\setminus\{i\}})P(X_{U\cup V})P(X_{U\cup
      V\setminus\{i\}}, Y)}\mid X\right ] \\
& = \ExpModel \left [\log \frac
  {\Prob(X_i,X_V|X_{U\setminus\{i\}},Y)}{\Prob(X_i|X_{U\setminus\{i\}},Y)
    \Prob(X_V|X_{U\setminus\{i\}},Y)} - \log \frac {\Prob
    (X_i,X_V|X_{U\setminus\{i\}})}{\Prob (X_i|X_{U\setminus\{i\}})
    \Prob (X_V|X_{U\setminus\{i\}})}\mid X\right ].
\end{align*}
Taking an expectation over $X$ at both sides, we can upper bound the
expected error by
\begin{align*}
\Exs |\tilde \phi_X^{[d]}(i) - \phi_X(i)| & \leq \frac{1}{d}
\sum_{U\in \mathcal C}\sum_{A:U_S(A) = U} \frac
    {1}{\binom{d-1}{|A|-1}}(\Exs \left | \log \frac
    {\Prob(X_i,X_{V_S(A)}|X_{U\setminus
        \{i\}},Y)}{\Prob(X_i|X_{U\setminus
        \{i\}},Y)\Prob(X_{V_S(A)}|X_{U\setminus \{i\}},Y)}\right |\\ &\qquad
    \qquad \qquad \qquad \qquad \quad +\Exs \left |\log \frac {\Prob
      (X_i,X_{V_S(A)}|X_{U\setminus \{i\}})}{\Prob (X_i|X_{U\setminus
        \{i\}})\Prob (X_{V_S(A)}|X_{U\setminus \{i\}})}\right |) \\
& = \frac{1}{d} \sum_{U\in \mathcal C}\sum_{A:U_S(A) = U}\frac
           {1}{\binom{d-1}{|A|-1}}(I_a(X_i;X_{V_S(A)}|X_{U\setminus
             \{i\}},Y)+I_a(X_i;X_{V_S(A)}|X_{U\setminus \{i\}}))\\
& \leq 2 \varepsilon.
\end{align*}
% \red{ big absolute value, big mid, big paranthesis and blocks. finally martin's comments.}
Let $R(U) \defn [d] - U\cup\{\max(u-1,1),\min(u+l+1,d)\}$. If we have
$X_i\independent X_{R(U)} | X_{U\setminus \{i\}}$ and $X_i\independent
X_{R(U)} | X_{U\setminus \{i\}},Y$ for any $U\subset [d]$, then
$\varepsilon=0$, which implies $\Exs |\tilde \phi_X^{[d]}(i) -
\phi_X(i)|=0$. Therefore, we have $\tilde \phi_X^{[d]}(i) = \phi_X(i)$
almost surely.

%%%%%%%%%%%%%%%%%%%%%%%%%%%%%%%%%%%%%%%%%%%%%%%%%%%%%%%%%%%%%%%%%%%%%%%%%%

\paragraph{Case 2:}  We now turn to the
general case $S\subset \nbhd{k}{i}\subset [d]$. Similar as above, we can show
\begin{align*}
\Exs |\tilde \phi_X^k(i) - \hat \phi_X^k(i)|\leq 2\varepsilon.
\end{align*}
Based on Theorem~1, we have
\begin{align*}
\Exs |\hat \phi_X^k(i) - \phi_X(i)| & \leq 4 \varepsilon.
\end{align*} 
Applying the triangle yields $\Exs |\tilde \phi_X^k(i) - \phi_X(i)|
\leq 6 \varepsilon$, which establishes the claim.

\end{document}